\documentclass{article}

\usepackage[final,nonatbib]{neurips_2021}

\usepackage[utf8]{inputenc} 
\usepackage[T1]{fontenc}    
\usepackage[colorlinks]{hyperref}       
\usepackage{url}            
\usepackage{booktabs}       
\usepackage{amsfonts}       
\usepackage{nicefrac}       
\usepackage{microtype}      
\usepackage{xcolor}         

\usepackage{listings}

\definecolor{codegreen}{rgb}{0,0.6,0}
\definecolor{codegray}{rgb}{0.5,0.5,0.5}
\definecolor{codepurple}{rgb}{0.58,0,0.82}
\definecolor{backcolour}{rgb}{0.95,0.95,0.92}

\lstdefinestyle{mystyle}{
  backgroundcolor=\color{backcolour},   commentstyle=\color{codegreen},
  keywordstyle=\color{magenta},
  numberstyle=\tiny\color{codegray},
  stringstyle=\color{codepurple},
  basicstyle=\ttfamily\footnotesize,
  breakatwhitespace=false,         
  breaklines=true,                 
  captionpos=b,                    
  keepspaces=true,                 
  numbers=left,                    
  numbersep=5pt,                  
  showspaces=false,                
  showstringspaces=false,
  showtabs=false,                  
  tabsize=2
}

\usepackage{graphicx}
\usepackage{comment}
\usepackage[caption=false,font=footnotesize]{subfig}
\usepackage{amsmath}
\usepackage[english]{babel}
\usepackage[linesnumbered,ruled,vlined]{algorithm2e}
\usepackage{algpseudocode}
\usepackage{amssymb}
\usepackage{amsthm}
\usepackage{bm}
\usepackage{units}
\usepackage{wrapfig,booktabs}
\usepackage{bbm}

\usepackage{caption}
\newcommand{\reals}{{\mathbb{R}}}
\newcommand{\integers}{{\mathbb{Z}}}

\newcommand{\bfu}[1]{\underline{\bf{#1}}}

\newcommand{\mtx}[1]{\mathbf{#1}}
\newcommand{\vc}[1]{\mathbf{#1}}
\newcommand{\tnsr}[1]{\mathsf{#1}}

\newcommand{\calA}{{\cal A}}

\newcommand{\calW}{{\cal W}}
\newcommand{\calK}{{\cal K}}
\newcommand{\todo}[1]{\textcolor{orange}{#1}}
\newtheorem{theorem}{Theorem}

\newtheorem{lemma}{Lemma}
\newtheorem*{lemma2}{Lemma}
\newtheorem{property}{Property}
\newtheorem*{theorem2}{Theorem}
\newcommand{\ns}[1]{{\color{red}NS: #1}}

\DeclareMathOperator*{\argmax}{arg\,max}
\DeclareMathOperator*{\argmin}{arg\,min}

\SetKwInput{KwInput}{Input}                
\SetKwInput{KwOutput}{Output}              

\SetCommentSty{mycommfont}
\title{Generalized Depthwise-Separable Convolutions for Adversarially Robust and Efficient Neural Networks}

\author{%
  Hassan Dbouk \& Naresh R. Shanbhag \\
  Department of Electrical and Computer Engineering\\
  University of Illinois at Urbana-Champaign\\
  Urbana, IL 61801 \\
  \texttt{\{hdbouk2,shanbhag\}@illinois.edu} \\
}

\begin{document}

\maketitle

\begin{abstract}
Despite their tremendous successes, convolutional neural networks (CNNs) incur high computational/storage costs and are vulnerable to adversarial perturbations. Recent works on robust model compression address these challenges by combining model compression techniques with adversarial training. But these methods are unable to improve throughput (frames-per-second) on real-life hardware while simultaneously preserving robustness to adversarial perturbations. To overcome this problem, we propose the method of Generalized Depthwise-Separable (GDWS) convolution -- an \textit{efficient, universal, post-training} approximation of a standard 2D convolution. GDWS dramatically improves the throughput of a standard pre-trained network on real-life hardware while preserving its robustness. Lastly, GDWS is scalable to large problem sizes since it operates on pre-trained models and doesn't require any additional training. We establish the optimality of GDWS as a 2D convolution approximator and present exact algorithms for constructing optimal GDWS convolutions under complexity and error constraints. We demonstrate the effectiveness of GDWS via extensive experiments on CIFAR-10, SVHN, and ImageNet datasets. Our code can be found at \url{https://github.com/hsndbk4/GDWS}.

\end{abstract}

\section{Introduction} \label{sec:intro}

Nearly a decade of research after the release of AlexNet \cite{krizhevsky2012imagenet} in 2012, convolutional neural networks (CNNs) have unequivocally established themselves as the \textit{de facto} classification algorithm for various machine learning tasks \cite{he2016deep,tan2019efficientnet,fastrcnn}. The tremendous success of CNNs is often attributed to their unrivaled ability to extract correlations from large volumes of data, allowing them to surpass human level accuracy on some tasks such as image classification \cite{he2016deep}. 

Today, the deployment of CNNs in safety-critical Edge applications is hindered due to their \textbf{high computational costs} \cite{he2016deep,sakr2017analytical,sakr2018analytical} and their \textbf{vulnerability} to adversarial samples \cite{szegedy2013intriguing,goodfellow2014explaining,ilyas2019adversarial}. Traditionally, those two problems have been addressed in isolation. Recently, very few bodies of works \cite{lin2018defensive,sen2019empir,admm,ATMC,sehwag2020hydra,NAS} have addressed the daunting task of designing both \textbf{efficient} and \textbf{robust} CNNs. A majority of these methods focus on model compression, i.e. reducing the storage requirements of CNNs. None have demonstrated their real-time benefits in hardware. For instance, Fig.~\ref{fig:comp-fps} shows recent robust pruning works HYDRA \cite{sehwag2020hydra} and ADMM \cite{admm} achieve high compression ratios (up to $97\times$) but either \textit{fail} to achieve high throughput measured in frames-per-second (FPS) or \textit{compromise} significantly on robustness. Furthermore, the overreliance of current robust complexity reduction techniques on adversarial training (AT) \cite{trades,madry2018towards} increases their training time significantly (Fig.~\ref{fig:comp-time}). This prohibits their application to complex ImageNet scale problems with stronger attack models, such as union of norm-bounded perturbations \cite{maini2020adversarial}. Thus, there is critical need for methods to design deep nets that are both adversarially robust and achieve high throughput when mapped to real hardware. 

To address this need, we propose \textbf{Generalized Depthwise-Separable (GDWS)} convolutions, a \textit{universal post-training} approximation of a standard 2D convolution that dramatically improves the real hardware FPS of pre-trained networks (Fig.~\ref{fig:comp-fps}) while preserving their robust accuracy. Interestingly, we find GDWS applied to un-pruned robust networks  simultaneously achieves higher FPS and higher robustness than robust pruned models obtained from current methods. This in spite of GDWS's compression ratio being smaller than those obtained from robust pruning methods. Furthermore, GDWS easily scales to large problem sizes since it operates on pre-trained models and doesn't require any additional training.

\textbf{Our contributions}:
\begin{enumerate}
    \item We propose GDWS, a novel convolutional structure that can be seamlessly mapped onto off-the-shelf hardware and accelerate pre-trained CNNs significantly while maintaining robust accuracy.
    \item We show that the error-optimal and complexity-optimal GDWS approximations of any pre-trained standard 2D convolution can be obtained via greedy polynomial time algorithms, thus eliminating the need for any expensive training.
    \item We apply GDWS to a variety of networks on CIFAR-10, SVHN, and ImageNet to simultaneously achieve higher robustness \textit{and} higher FPS than existing robust complexity reduction techniques, while incurring no extra training cost. 
    \item We demonstrate the versatility of GDWS by using it to design efficient CNNs that are robust to union of $(\ell_\infty,\ell_2,\ell_1)$ perturbation models. To the best of our knowledge, this is the first work that proposes efficient and robust networks to the union of norm-bounded perturbation models.
\end{enumerate}
\begin{figure}[t]
  \centering
    \subfloat[]{\includegraphics[height=5.5cm]{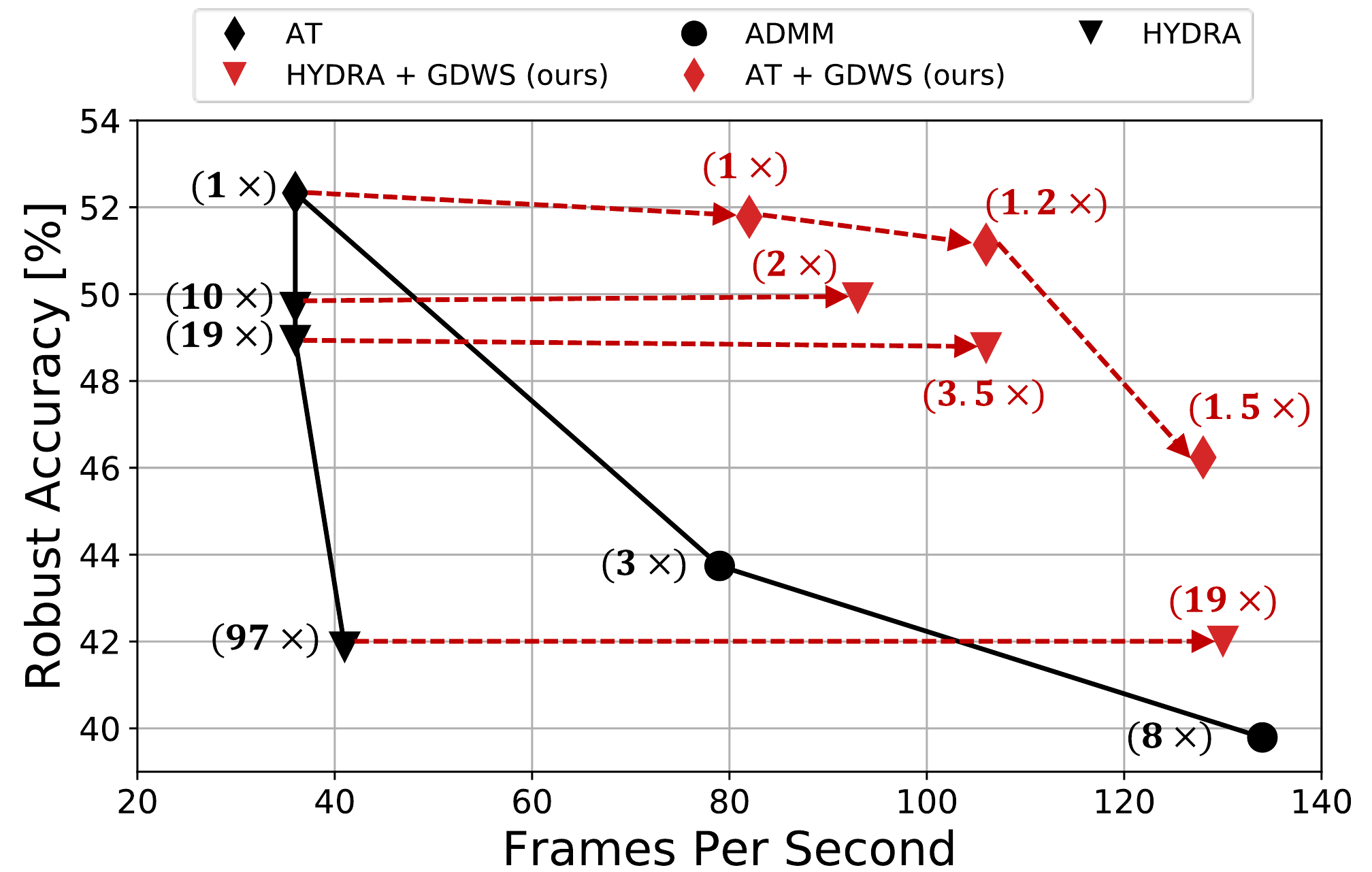}\label{fig:comp-fps}}%
    \qquad%
    \subfloat[]{\includegraphics[height=5.5cm]{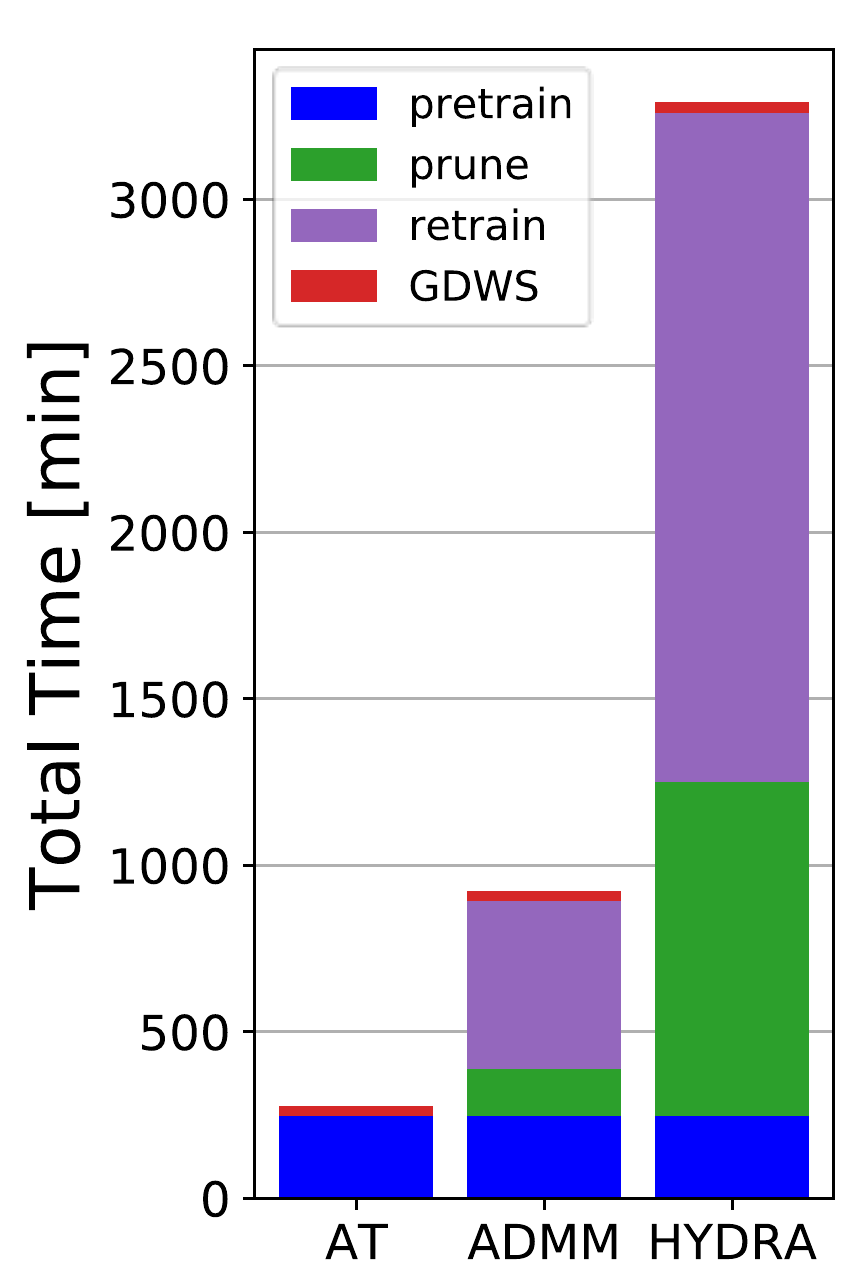}\label{fig:comp-time}}%
    \label{fig:1}
  \caption{Performance of existing robust pruning works (HYDRA \cite{sehwag2020hydra} and ADMM \cite{admm}) and the proposed GDWS with VGG-16 on CIFAR-10, captured by: (a) robust accuracy against $\ell_\infty$-bounded perturbations vs frames-per-second measured on an NVIDIA Jetson Xavier, and (b) total time required to implement these methods measured on a single NVIDIA 1080 Ti GPU. To ensure a fair comparison, the same AT baseline (obtained from \cite{sehwag2020hydra}) is used for all methods. The compression ratio of each method, highlighted in parenthesis, is with respect to the AT baseline.}
\end{figure}

\section{Background and Related Work}

The problem of designing \textbf{efficient} and \textbf{robust} CNNs, though crucial for safety-critical Edge applications, is not yet well understood. Very few recent works have addressed this problem \cite{lin2018defensive, sen2019empir, admm,ATMC, sehwag2020hydra, NAS}. We cluster prior works into the following categories:

\textbf{Quantization} Reducing the complexity of CNNs via model quantization in the absence of any adversary is a well studied problem in the deep learning literature \cite{sakr2017analytical,sakr2018analytical,choi2018pact,zhang2018lq,hubara2016binarized,rastegari2016xnor,dbouk2020dbq}. The role of quantization on adversarial robustness was studied in Defensive Quantization (DQ) \cite{lin2018defensive} where it was observed that conventional \textit{post-training} fixed-point quantization makes networks more \textit{vulnerable} to adversarial perturbations than their full-precision counterparts. EMPIR \cite{sen2019empir} also leverages extreme model quantization (up to 2-bits) to build an ensemble of efficient and robust networks. However, \cite{tramer2020adaptive} broke EMPIR by constructing attacks that fully leverage the model structure, i.e., adaptive attacks.  In contrast, GDWS is an orthogonal complexity reduction technique that preserves the base model's adversarial robustness and can be applied in conjunction with model quantization. 

\textbf{Pruning} The goal of pruning is to compress neural networks by zeroing out unimportant weights \cite{han2015deep,guo2016dynamic,zhang2018systematic,yang2017designing}. The structured pruning method in \cite{admm} combines the alternating direction method of multipliers (ADMM) \cite{zhang2018systematic} for parameter pruning within the AT framework \cite{madry2018towards} to design pruned and robust networks. The flexibility of ADMM enables it to achieve a high FPS on Jetson (as seen in Fig.~\ref{fig:comp-fps}) but suffers from a significant drop in robustness. ATMC \cite{ATMC} augments the ADMM framework \cite{admm} with model quantization and matrix factorization to further boost the compression ratio. On the other hand, unstructured pruning methods such as HYDRA \cite{sehwag2020hydra} prunes models via important score optimization \cite{ramanujan2020s}. 
However, HYDRA's high pruning ratios ($>90\%$) doesn't translate into real-time FPS improvements on off-the-shelf hardware and often requires custom hardware design to fully leverage their capabilities \cite{han2016eie}. GDWS is complementary to unstructured pruning methods, e.g., when applied to HYDRA, GDWS boosts the achievable FPS and achieves much higher robustness at iso-FPS when compared to structured (filter) pruning ADMM. 

\textbf{Neural Architecture Search} Resource-efficient CNNs can be designed by exploiting design intuitions such as depthwise separable (DWS) convolutions \cite{howard2017mobilenets,sandler2018mobilenetv2,huang2018condensenet,zhang2018shufflenet,iandola2016squeezenet,tan2019efficientnet}. While neural architecture search (NAS) \cite{zoph2018learning,real2019regularized} automates the process, it requires massive compute resources, e.g., thousands of GPU hours for a single network. Differentiable NAS \cite{liu2018darts} and one-shot NAS \cite{bender2018understanding} drastically reduce the cost of this search. In \cite{NAS}, a one-shot NAS framework \cite{bender2018understanding} is combined with the AT framework \cite{madry2018towards} to search for robust network architectures, called RobNets. RobNets achieve slightly higher robustness than existing networks with less storage requirements. In this work, we show that applying GDWS to existing architectures, e.g., WideResNet-28-4, achieves significantly higher FPS than RobNet, at iso-robustness and model size. 


\section{Generalized Depthwise-Separable Convolutions} \label{sec:gdws}
In this section, we introduce GDWS convolutions and develop error-optimal and complexity-optimal  GDWS approximations of standard 2D convolution. These optimal approximations are then employed to construct GDWS networks from any pre-trained robust CNN built from standard 2D convolutions. 

\textbf{Notation:} A $(C,K,M)$ standard 2D convolution operates on an input feature map $\tnsr{X}\in\reals^{C\times H\times W}$ via $M$ filters (also referred to as kernels or output channels) each consisting of $C$ channels each of dimension $K\times K$ to generate an output feature map $\tnsr{Y}\in\reals^{M\times H'\times W'}$. 

\textbf{2D Convolution as Matrix Multiplication:} The $M$ filters can be viewed as vectors $\{\vc{w}_i\}_{i=1}^M\in\reals^{CK^2}$ obtained by vectorizing the $K^2$ elements within a channel and then across the $C$ channels. The resulting weight matrix $\mtx{W} \in \reals^{M\times CK^2}$ is constructed by stacking these filter vectors, i.e., $\mtx{W} = [\vc{w}_1 | \vc{w}_2 | ... | \vc{w}_M]^\text{T}$. 

From an operational viewpoint, the matrix $\mtx{W}$ can be used to compute the 2D convolution via Matrix Multiplication (MM)  with the input matrix $\mtx{X}=\Psi(\tnsr{X}) \in \reals^{CK^2\times H'W'}$:
\begin{equation} \label{eq:conv-eq}
    \mtx{Y} = \mtx{W}\mtx{X} = \mtx{W}\Psi(\tnsr{X})
\end{equation}
where $\Psi$ is an unrolling operator that generates all $H'W'$ input feature map slices and stacks them in matrix format. The resultant output matrix $\mtx{Y}\in\reals^{M\times H'W'}$ can be reshaped via the operator $\Phi$ to retrieve $\tnsr{Y}=\Phi(\mtx{Y})$. The computational complexity of \eqref{eq:conv-eq} in terms of multiply-accumulate (MAC) operations is given by:
\begin{equation}
    H'W'MCK^2
\end{equation}
The reshaping operators $\Phi$ and $\Psi$ are only for notational convenience and are computation-free.
\subsection{GDWS Formulation}

\begin{figure}[t]
  \centering
    \includegraphics[width=\columnwidth]{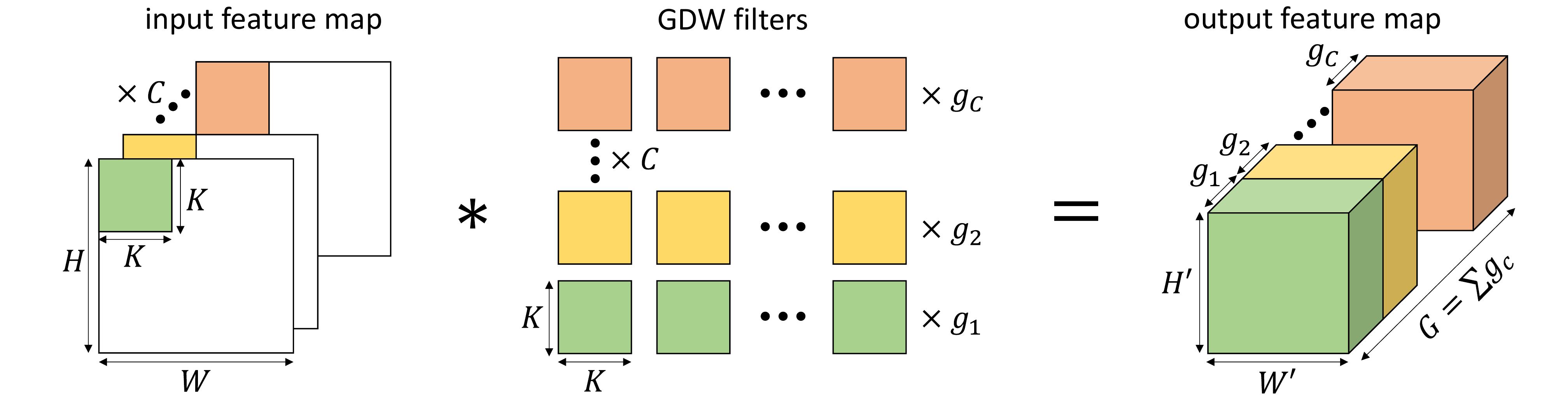}

  \caption{The $(C,K,\vc{g})$ generalized depthwise (GDW) convolution operation. A standard depthwise (DW) convolution is obtained by setting $g_c=1\ \forall c\in [C]$.}    \label{fig:gdw}
\end{figure}
\textbf{Definition:} 
A $(C,K,\vc{g},M)$ GDWS convolution is parameterized by the \emph{channel distribution vector} $\vc{g}\in\integers_+^{C}$ in addition to the parameters $(C,K,M)$ of a standard 2D convolution. A GDWS convolution is composed of a $(C,K,\vc{g})$ \emph{Generalized Depthwise} (GDW) convolution and a $(G,1,M)$ standard pointwise (PW) convolution where $G=\sum g_c$ with $c\in[C]$\footnote{we use the notation $[C] = \{1, 2, ..., C\}$ for brevity.}.

A $(C,K,\vc{g})$ GDW convolutional layer (Fig.~\ref{fig:gdw}) operates on an input feature map $\tnsr{X}\in \reals^{C\times H\times W}$ by convolving the $c^\text{th}$ channel with $g_c\in\integers_+$ depthwise $K\times K$ filters to produce a total of $G$ intermediate output channels each of size $H'\times W'$. The $(G,1,M)$ PW layer operates on the intermediate output feature map by convolving it with $M$ filters of size $1\times 1$, thus producing the output feature map $\tnsr{Y} \in \reals^{M\times H' \times W'}$. 

\textbf{Relation to DWS:} Setting $g_c = 1\ \forall c\in [C]$ reduces the GDWS convolution to the standard DWS convolution popularized by \cite{howard2017mobilenets}. Thus, GDWS generalizes DWS by allowing for more than one ($g_c\geq1$) depthwise filters per channel. This simple generalization relaxes DWS's highly constrained structure enabling accurate approximations of the 2D convolution. Thus, GDWS when applied to pre-trained models preserves its original behavior and therefore its natural and robust accuracy. Furthermore, GDWS achieves high throughput since it exploits the same hardware features that enable networks with DWS to be implemented efficiently. One might ask: why not use DWS on pre-trained models? Doing so will result in very high approximation errors. In fact, in Section~\ref{ssec:results}, we show that applying GDWS to a pre-trained complex network such as ResNet-18 achieves better robust accuracy than MobileNet trained from scratch, while achieving similar FPS.

\textbf{GDWS Complexity:} The total number of MAC operations required by GDWS convolutions is:
\begin{equation}\label{eq:gdws-comp}
    H'W'\Big(\sum_{c=1}^Cg_c(K^2 + M)\Big) = H'W'G(K^2+M)
\end{equation}
Thus, replacing standard 2D convolutions with GDWS convolutions results in a complexity reduction by a factor of $\frac{G(K^2+M)}{CK^2M}$.

\subsection{Properties of GDWS Convolutions}
We present properties of the GDWS weight matrix $\mtx{W}$ that will be vital for developing the optimal approximation procedures.
\begin{property}\label{prop:gdws_mtx} The weight matrix $\mtx{W}\in \reals^{M\times CK^2}$ of a $(C,K,\vc{g},M)$ GDWS convolution can be expressed as:
\begin{equation}
    \mtx{W} = \mtx{W}_{\text{\normalfont P}}\mtx{W}_{\text{ \normalfont D}}
\end{equation}
where $\mtx{W}_{\text{\normalfont P}} \in \reals^{M\times G}$ and $\mtx{W}_{\text{\normalfont D}}\in \reals^{G\times CK^2}$ are the weight matrices of the PW and GDW convolutions, respectively. 
\end{property}
Property~\ref{prop:gdws_mtx} implies that any GDWS convolution has an equivalent 2D convolution whose weight matrix is the product of $\mtx{W}_\text{P}$ and $\mtx{W}_\text{D}$, where $\mtx{W}_\text{P}$ is a regular convolution weight matrix with $K=1$ and $\mtx{W}_\text{D}$ has the following property:
\begin{property}\label{prop:gdw_mtx} The weight matrix $\mtx{W}_\text{\normalfont D}\in \reals^{G\times CK^2}$ of a $(C,K,\vc{g})$ GDW convolution has a block-diagonal structure. Specifically, $\mtx{W}_\text{\normalfont D}$ is a concatenation of $C$ sub-matrices where each sub-matrix $\mtx{W}_{\text{\normalfont D},c} \in \reals^{G\times K^2}$ has at most $g_c$ non-zero consecutive rows, starting at row index $1+\sum_{k=1}^{c-1}g_k$ as shown in Fig.~\ref{fig:gdw_mtx}. 
\end{property}
This structure is due to the fact that input channels are convolved independently with at most $g_c$ depthwise filters per channel.
Finally, let $\mtx{W} = [\mtx{W}_1|\mtx{W}_2| ...| \mtx{W}_C ]$ be represented as the concatenation of $C$ sub-matrices, then combining Properties~\ref{prop:gdws_mtx} \& \ref{prop:gdw_mtx} establishes the following lemma:
\begin{lemma}\label{lemm:gdws_mtx_concat} The weight matrix of a $(C,K,\vc{g},M)$ GDWS convolution can be expressed as the concatenation of $C$ sub-matrices $\mtx{W}_c$ where $\text{\normalfont rank}(\mtx{W}_c) \leq \text{\normalfont{min}}(g_c,K^2)\ \forall c\in[C]$.
\end{lemma}
The reason for this is that each sub-matrix $\mtx{W}_c$ can be expressed as the sum of $g_c$ rank 1 matrices of size $M\times K^2$. A detailed proof of Lemma~\ref{lemm:gdws_mtx_concat} can be found in the Appendix. A major implication of Lemma~\ref{lemm:gdws_mtx_concat} is that any 2D standard convolution is equivalent to a GDWS convolution with $g_c =K^2$ $\forall c$\footnote{Typical CNNs satisfy $K^2<M$, e.g., $K=3$ and $M\geq 16$}. Hence, in the rest of this paper we will assume $g_c \leq K^2$ when we approximate 2D convolutions with GDWS.
\begin{figure}[t]
  \centering
    \includegraphics[width=0.9\columnwidth]{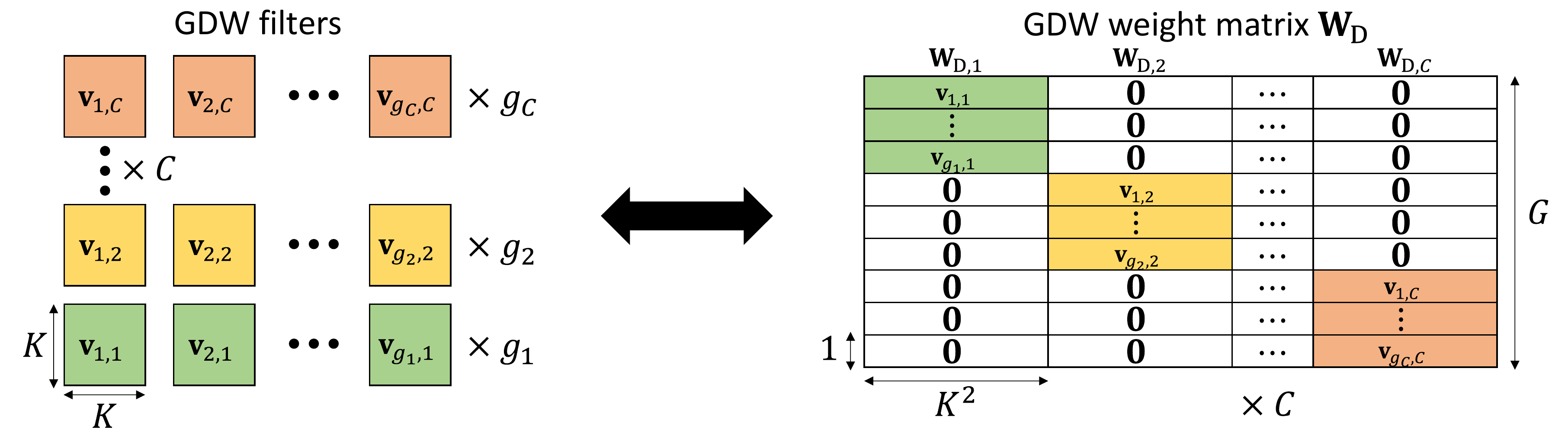}

  \caption{The weight matrix representation of a GDW convolution operation. All the $\vc{v}_{i,j}$ vectors are row vectors with $K^2$ elements.}
      \label{fig:gdw_mtx}
\end{figure}

\subsection{Optimal GDWS Approximation Methods}
We wish to approximate a standard 2D convolution with weight matrix $\mtx{W}\in \reals^{M\times CK^2}$ with a GDWS convolution with weight matrix $\mtx{Q}\in \reals^{M\times CK^2}$ to minimize the weighted approximation error defined as:
\begin{equation}\label{eq:err-exp}
    e(\mtx{W},\mtx{Q},\bm{\alpha}) = \sqrt{\sum_{c=1}^C \alpha_c||\mtx{W}_c-\mtx{Q}_c||^2_\text{F}}
\end{equation}
where $||.||_\text{F}$ denotes the Frobenius norm of a matrix and $\bm{\alpha}\in\reals_+^C$ is a vector of positive weights. Setting $\alpha_c=1\ \forall c$ simplifies \eqref{eq:err-exp} to the Frobenius norm of the error matrix $\mtx{W}-\mtx{Q}$. Furthermore, from \eqref{eq:gdws-comp}, one can upper bound the complexity of the GDWS approximation $\mtx{Q}$ via an upper bound on $G=\sum g_c$ where the $g_c$'s are obtained from Lemma~\ref{lemm:gdws_mtx_concat}. 

Based on the GDWS properties and Lemma~\ref{lemm:gdws_mtx_concat}, we state the following error-optimal approximation theorem:
\begin{theorem}\label{thm:approximation_comp} Given a $(C,K,M)$ standard 2D convolution with weight matrix $\mtx{W}$, the $(C,K,\vc{g},M)$ GDWS approximation with weight matrix $\hat{\mtx{W}}$ that minimizes the error in \eqref{eq:err-exp} subject to $\sum g_c = G\leq \gamma$ (for some $\gamma \in \integers_+$), can be obtained in polynomial time via Algorithm~\ref{alg:approx1}. 
\end{theorem}
That is:
\begin{equation}\label{eq:thm1}
    \hat{\mtx{W}} = \argmin_{\mtx{Q}:\ G\leq \gamma}  e(\mtx{W},\mtx{Q},\bm{\alpha})
\end{equation}
can be solved for any weight error vector $\bm{\alpha}\in\reals_+^C$ in polynomial time. While Theorem~\ref{thm:approximation_comp} shows that the optimal GDWS approximation under a complexity constraint can be solved efficiently, a similar result can be obtained for the reverse setting shown next.
\begin{theorem}\label{thm:approximation_err} Given a $(C,K,M)$ standard 2D convolution with weight matrix $\mtx{W}$, the $(C,K,\vc{g},M)$ GDWS approximation with weight matrix $\hat{\mtx{W}}$ that minimizes the complexity in \eqref{eq:gdws-comp} subject to $e(\mtx{W},\mtx{Q},\bm{\alpha})\leq \beta$ (for some $\beta \geq 0$), can be constructed in polynomial time via Algorithm~\ref{alg:approx2}. 
\end{theorem}
That is:
\begin{equation}\label{eq:thm2}
    \hat{\mtx{W}} = \argmin_{\mtx{Q}:\ e(\mtx{W},\mtx{Q},\bm{\alpha})\leq \beta}  \sum_{c=1}^Cg_c
\end{equation}
can be solved for any weight error vector $\bm{\alpha}\in\reals_+^C$ in polynomial time. Proofs of Theorems~\ref{thm:approximation_comp} \& \ref{thm:approximation_err} can be found in the Appendix.

\begin{minipage}[b]{0.47\textwidth}
\begin{algorithm}[H]
\label{alg:approx1}
\DontPrintSemicolon
  
  \KwInput{A $(C,K,M)$ convolution $\mtx{W}$, weight error vector $\bm{\alpha}$, and constraint $\gamma \in \integers_+$.}
  \KwOutput{A $(C,K,\vc{g},M)$ GDWS convolution $\hat{\mtx{W}}$, satisfying $\sum g_c \leq \gamma$.}
  Compute SVDs of $\mtx{W}_c = \sum_{i=1}^{r_c}\sigma_{i,c} \vc{u}_{i,c} \vc{v}_{i,c}^{\text{\normalfont T}}$\;
  Initialize $\vc{g}=\vc{0}$\;
  \While{$\sum g_c < \gamma$}{
    $c' = \argmax_c \alpha_c \sigma^2_{g_c+1,c}$ \tcp*{$g_c < r_c$} 
    $g_{c'} \leftarrow g_{c'} + 1$\;
  }
  Compute $\hat{\mtx{W}}_c$ via truncated SVD of $\mtx{W}_c$ with rank $g_c$: $\hat{\mtx{W}}_c = \sum_{i=1}^{g_c}\sigma_{i,c} \vc{u}_{i,c} \vc{v}_{i,c}^{\text{\normalfont T}}$ \;
  Construct $\hat{\mtx{W}} = [\hat{\mtx{W}}_1|...|\hat{\mtx{W}}_C]$\;
\caption{(\texttt{MEGO}) \underline{M}inimum \underline{E}rror Complexity-constrained \underline{G}DWS \underline{O}ptimal Approximation}
\end{algorithm}
\end{minipage}
\hfill
\begin{minipage}[b]{0.48\textwidth}
\begin{algorithm}[H]
\label{alg:approx2}
\DontPrintSemicolon
  
  \KwInput{A $(C,K,M)$ convolution $\mtx{W}$, weight error vector $\bm{\alpha}$, and constraint $\beta \geq 0$.}
  \KwOutput{A $(C,K,\vc{g},M)$ GDWS convolution $\hat{\mtx{W}}$, satisfying $e\leq \beta$.}
   Compute SVDs of $\mtx{W}_c = \sum_{i=1}^{r_c}\sigma_{i,c} \vc{u}_{i,c} \vc{v}_{i,c}^{\text{\normalfont T}}$\;
  Initialize $g_c=r_c$, $b=0$, $c' = \argmin_c \alpha_c \sigma^2_{r_c,c}$, $h = \alpha_{c'} \sigma^2_{r_{c'},c'}$\;
   \While{$b + h < \beta$}{
   $b\leftarrow b + h$ and $g_{c'} \leftarrow g_{c'} - 1$\;   
    $c' = \argmin_c \alpha_c \sigma^2_{g_c,c}$ \tcp*{$g_c >1$} 
    $h = \alpha_{c'}\sigma^2_{r_{c'},c'}$\;
  }
  Compute $\hat{\mtx{W}}_c$ via truncated SVD of $\mtx{W}_c$ with rank $g_c$: $\hat{\mtx{W}}_c = \sum_{i=1}^{g_c}\sigma_{i,c} \vc{u}_{i,c} \vc{v}_{i,c}^{\text{\normalfont T}}$ \;
  Construct $\hat{\mtx{W}} = [\hat{\mtx{W}}_1|...|\hat{\mtx{W}}_C]$\;
\caption{(\texttt{LEGO}) \underline{L}east Complex \underline{E}rror-constrained \underline{G}DWS \underline{O}ptimal Approximation}
\end{algorithm}
\end{minipage}

\subsection{Constructing GDWS Networks}\label{ssec:gdws-networks}
When confronted with a CNN with $L$ convolutional layers, the question arises: \textit{How to assign resources (in terms of complexity) amongst the $L$ layers such that the robustness of the CNN is minimally compromised?} 
To answer this question, we compute per-layer weight error vectors $\bm{\alpha}_l$, such that the computed error in \eqref{eq:err-exp} weighs how different sub-matrices affect the final output of the CNN. 

Let $f: \reals^D \rightarrow \reals^{N}$ be a pre-trained CNN for an $N$-way classification problem with $L$ convolutional layers parameterizd by weight matrices $\mtx{W}^{(l)} \in \reals^{M_l\times C_lK_l^2}$. The CNN $f$ operates on a $D$-dimensional input vector $\vc{x}$ to produce a vector $\vc{z}=f(\vc{x})$ of soft outputs or logits. Denote by $n_x \in [N]$ the predicted class label associated with $\vc{x}$, and define $\delta_{x,j} = z_j - z_{n_x}$ to be the soft output differences $\forall j\in[N]\setminus\{n_x\}$. 

Inspired by \cite{sakr2017analytical,sakr2018analytical}, we propose a simple yet effective method for computing the per-layer weight error vectors as follows:
\begin{equation}\label{eq:alpha}
    \alpha_{c,l} = \frac{1}{M_lK_l^2}\mathbb{E} \left[\sum_{\substack{j=1\\ j \neq n_x}}^N \frac{||\mtx{D}^{(c,l)}_{x,j}||_{\text{F}}^2}{2\delta_{x,j}^2}\right] \ \ \ \forall l \in [L],\ \forall c\in [C_l]
\end{equation}
where $\mtx{D}^{(c,l)}_{x,j} \in \reals^{M_l \times K_l^2}$ is the derivative of $\delta_{x,j}$ w.r.t. the sub-matrix $\mtx{W}^{(l)}_c$. The expectation is taken over the input vector $\vc{x}$. Equation~\eqref{eq:alpha} can be thought of as the expected noise gain from a particular channel in a particular layer to the network output required to flip its decision. The Appendix provides a detailed rationale underlying \eqref{eq:alpha}.

Computation of \eqref{eq:alpha} can be simplified by obtaining an estimate of the mean over a small batch of inputs sampled from the training set and by  leveraging software frameworks such as PyTorch \cite{paszke2017automatic} that automatically take care of computing $\mtx{D}^{(c,l)}_{x,j}$. Algorithm~\ref{alg:network} summarizes the steps required to approximate any pre-trained CNN with an equivalent CNN utilizing GDWS convolutions. Unless specified otherwise, all the results in this paper are obtained via Algorithm~\ref{alg:network}.
\begin{minipage}{0.98\textwidth}
    \begin{algorithm}[H]
\label{alg:network}
\DontPrintSemicolon
  
  \KwInput{CNN $f$ with convolutional layers $\{\mtx{W}^{(l)}\}$, $\{\bm{\alpha}_l\}$ computed via \eqref{eq:alpha}, and constraint $\beta \geq 0$.}
  \KwOutput{CNN $\hat{f}$ with GDWS convolutions $\{\hat{\mtx{W}}^{(l)}\}$}
   $\hat{f} = f$\tcp*{Initialize $\hat{f}$}
   \For{$l \in \{1, ..., L\}$}
   {
   		$\hat{\mtx{W}}^{(l)} = \texttt{LEGO}(\mtx{W}^{(l)}, \bm{\alpha}_l, \beta)$ \tcp*{solve via Algorithm~\ref{alg:approx2}}
   		Decompose $\hat{\mtx{W}}^{(l)}$ into GDW and PW convolutions via Property~\ref{prop:gdws_mtx} and Lemma~\ref{lemm:gdws_mtx_concat} \;
   		Replace the $l^{\text{th}}$ convolution layer in $\hat{f}$ with GDW and PW convolutions\;
   }

\caption{Constructing GDWS networks}
\end{algorithm}

\end{minipage}

\section{Experiments}
\subsection{Evaluation Setup}

We measure the throughput in FPS by mapping the networks onto an NVIDIA Jetson Xavier via native PyTorch \cite{paszke2017automatic} commands. 
We experiment with VGG-16 \cite{simonyan2014very}, ResNet-18\footnote{For CIFAR-10 and SVHN, we use the standard pre-activation version of ResNets.} \cite{he2016deep}, ResNet-50, and WideResNet-28-4 \cite{zagoruyko2016wide} network architectures, and report both natural accuracy ($\calA_\text{nat}$) and robust accuracy ($\calA_{\text{rob}}$). Following standard procedure, we report $\calA_{\text{rob}}$ against $\ell_\infty$ bounded perturbations generated via PGD \cite{madry2018towards} with standard attack strengths: $\epsilon=8/255$ with PGD-100 for both CIFAR-10 \cite{cifar10} and SVHN \cite{svhn} datasets, and $\epsilon = 4/255$ with PGD-50 for the ImageNet \cite{russakovsky2015imagenet} dataset. Section~\ref{ssec:union} studies union of multiple perturbation models ($\ell_\infty, \ell_2, \ell_1$). In the absence of publicly released pre-trained models, we establish strong baselines using AT \cite{madry2018towards} following the approach of \cite{rice2020overfitting} which utilizes early stopping to avoid robust over-fitting. 
Details on the training/evaluation setup can be found in the Appendix.

\begin{table}[hp]
\caption{Comparison between RobNet \cite{NAS} and GDWS on the CIFAR-10 dataset. GDWS is applied to standard pre-trained models.}
\label{tab:ablation-cifar10}
\centering
\resizebox{0.6\columnwidth}{!}{
  \begin{tabular}{l|ccc|c}
    \toprule
    \bf{Models} & \bf{$\calA_{\text{nat}}\ [\%]$} & \bf{$\calA_{\text{rob}}\ [\%]$}    & \bf{Size $[\text{MB}]$}    & \bf{FPS}\\
    \midrule
    RobNet \cite{NAS}  &82.72& 52.23 &20.8& 5\\
    \midrule
    ResNet-50 & 84.21 & 53.05 & 89.7 &16\\
    + GDWS ($\beta=0.001$) & \bfu{83.72} & \bfu{52.94} & 81.9 &\bfu{37}\\
    \midrule
    WRN-28-4  & 84.00 & 51.80     & 22.3 & 17  \\
    + GDWS ($\beta=1\times 10^{-5}$) & \bfu{83.27} & \bfu{51.70} & 18.9 & \bfu{65}\\
    \midrule
    ResNet-18 & 82.41  & 51.55 & 42.6 & 28    \\
    + GDWS ($\beta=0.005$) & \bfu{81.17} & \bfu{50.98} & 29.1 & \bfu{104}\\
    \midrule
    VGG-16     & 77.49 & 48.92 &   56.2  &   36  \\
   + GDWS ($\beta=0.25$) & \bfu{77.17} & \bfu{49.56} & 28.7 &\bfu{129}\\

    \bottomrule
  \end{tabular}}
  \end{table}
\subsection{Results}\label{ssec:results}
\textbf{Ablation Study:}
We first show the effectiveness of GDWS on the CIFAR-10 datasets using four network architectures. Table~\ref{tab:ablation-cifar10} summarizes $\calA_\text{nat}$ and $\calA_\text{rob}$ as well as FPS and model size. It is clear that GDWS networks preserve robustness as both $\calA_\text{nat}$ and $\calA_\text{rob}$ are always within $\sim$1$\%$ of their respective baselines. The striking conclusion is that in spite of GDWS offering modest reductions in model size, it \textit{drastically} improves the FPS of the base network across diverse architectures. For instance, a ResNet-18 utilizing GDWS convolutions is able to run at 104 FPS compared to the baseline's 28 FPS (>250\% improvement) without additional training and without compromising on robust accuracy. In the Appendix, we explore the benefits of applying GDWS using both Algorithms~\ref{alg:approx1}~\&~\ref{alg:approx2}, provide more detailed results on CIFAR-10 and show that similar gains are observed with SVHN dataset.

\textbf{GDWS vs. RobNet:} In Table~\ref{tab:ablation-cifar10}, we also compare GDWS networks (obtained from standard networks) with a publicly available pre-trained RobNet model, the robust network architecture designed via the NAS framework in \cite{NAS}. Note that RobNet utilizes DWS convolutions which precludes the use of GDWS. However, despite the efficiency of DWS convolutions in RobNet, its irregular cell structure leads to extremely poor mapping on the Jetson as seen by its low 5 FPS. For reference, a standard WideResNet-28-4 (WRN-28-4) runs at 17 FPS with similar robustness and model size. Applying GDWS to the WideResNet-28-4 further increases the throughput to 65 FPS which is a 1200\% improvement compared to RobNet while maintaining robustness. This further supports our assertion that model compression alone does not lead to enhanced performance on real hardware.

\begin{wraptable}{R}{0.6\columnwidth}
  \caption{Comparison between GDWS and lightweight networks on the CIFAR-10 dataset. The GDWS numbers are from Table~\ref{tab:ablation-cifar10}.}
  \label{tab:comp-lightnets}
  \centering
  \resizebox{0.6\columnwidth}{!}{%
  \begin{tabular}{l|ccc|c}
    \toprule
    \bf{Models} & \bf{$\calA_{\text{nat}}\ [\%]$} & \bf{$\calA_{\text{rob}}\ [\%]$}    & \bf{Size $[\text{MB}]$} & \bf{FPS}\\
    \midrule
    ResNet-18 + GDWS & \bfu{81.17} & \bfu{50.98} & 29.1 & \bfu{104}    \\
    VGG-16 + GDWS  &77.17 & 49.56 & 28.7 &129 \\
    \midrule
    MobileNetV1 & 79.92& 49.08& 12.3& 125\\
    MobileNetV2 & 79.59& 48.55&8.5  & 70\\
    ResNet-18 (DWS) & 80.12  & 48.52 & 5.5 & 120\\
    ResNet-20 &74.82 & 47.00& 6.4& 125 \\
    \bottomrule
  \end{tabular}}
\end{wraptable}

\textbf{GDWS vs. Lightweight Networks}:
A natural question that might arise from this work: why not train lightweight networks utilizing DWS convolutions from scratch instead of approximating pre-traind complex networks with GDWS? In Table~\ref{tab:comp-lightnets}, we compare the performance of GDWS networks (obtained from Table~\ref{tab:ablation-cifar10}) vs. standard lightweight networks: MobileNetV1 \cite{howard2017mobilenets}, MobileNetV2 \cite{sandler2018mobilenetv2}, and ResNet-20 \cite{he2016deep}, as well as a DWS-version of the standard ResNet-18 trained from scratch on the CIFAR-10 dataset. We find that applying GDWS to a pre-trained complex network such as ResNet-18 achieves better $\calA_{\text{nat}}$ and $\calA_{\text{rob}}$ than all lightweight networks, while achieving DWS-like FPS and requiring no extra training despite offering modest reductions in model size. The only benefit of using lightweight networks is the much smaller model size compared to GDWS networks.

\begin{wraptable}{R}{0.6\columnwidth}
    \caption{Comparison between ADMM \cite{admm} and GDWS using VGG-16 and ResNet-18 on CIFAR-10.}
      \label{tab:comp-admm}
      \centering
        \resizebox{0.6\columnwidth}{!}{%
  \begin{tabular}{l|ccc|c}
    \toprule
    \bf{Models} & \bf{$\calA_{\text{nat}}\ [\%]$} & \bf{$\calA_{\text{rob}}\ [\%]$}    & \bf{Size $[\text{MB}]$}  & \bf{FPS}\\
    \midrule
    VGG-16 (AT from \cite{admm}) &77.45 & 45.78 & 56.2 &36 \\
    + GDWS ($\beta=0.5$) & \bfu{76.40}& \bfu{46.28} & 38.8 & \bfu{119}\\
    
    \midrule
    VGG-16 ($p=25\%$) & 77.88& 43.80& 31.6&  26 \\
    VGG-16 ($p=50\%$) & 75.33& 42.93& 14.0&  113\\
    VGG-16 ($p=75\%$) & 70.39& 41.07& 3.5&  174\\
    \bottomrule
    \toprule
     ResNet-18 (AT from \cite{admm}) &80.65 & 47.05 & 42.6 &  28 \\
    + GDWS ($\beta=0.75)$ &\bfu{79.13}& \bfu{46.15}& 30.4&  \bfu{105}\\
    \midrule
    ResNet-18 ($p=25\%$) &  81.61& 42.67& 32.1&  31 \\
    ResNet-18 ($p=50\%$) & 79.42& 42.23& 21.7&  60\\
    ResNet-18 ($p=75\%$) & 74.62& 43.23& 11.2&  74\\
    \bottomrule
  \end{tabular}}
\end{wraptable}
\textbf{GDWS vs. Structured Pruning:}
In Table~\ref{tab:comp-admm}, we compare GDWS with the robust structured pruning method ADMM \cite{admm} on CIFAR-10, using two networks: VGG-16 and ResNet-18. Due to the lack of publicly available pre-trained models, we use their released code to reproduce both the AT baselines, and the corresponding pruned models at different pruning ratios. The nature of structured pruning allows ADMM pruned networks ($p\geq 50\%$) to achieve both high compression ratios and significant improvement in FPS over their un-pruned baselines but at the expense of robustness and accuracy. For instance, a ResNet-18 with 75\% of its channels pruned results in a massive 7$\%$ (4$\%$) drop in $\calA_\text{nat}$ ($\calA_\text{rob}$) compared to the baseline even though it achieves a 160$\%$ improvement in FPS. In contrast, a post-training application of GDWS to the \textit{same} ResNet-18 baseline results in a massive 275$\%$ improvement in FPS while preserving both $\calA_\text{nat}$ and $\calA_\text{rob}$ within 1$\%$ of their baseline values. Thus, despite achieving modest compression ratios compared to ADMM, GDWS achieves comparable improvements in FPS without compromising robustness.

\begin{wraptable}{R}{0.6\columnwidth}
    \caption{Comparison between HYDRA \cite{sehwag2020hydra} and GDWS using VGG-16 and WRN-28-4 on CIFAR-10.}
      \label{tab:comp-hydra}
      \centering
    \resizebox{0.6\columnwidth}{!}{%
  \begin{tabular}{l|ccc|c}
    \toprule
    \bf{Models} & \bf{$\calA_{\text{nat}}\ [\%]$} & \bf{$\calA_{\text{rob}}\ [\%]$}    & \bf{Size $[\text{MB}]$} &\bf{FPS}\\
    \midrule
    VGG-16 (AT from \cite{sehwag2020hydra})&82.72 & 51.93 & 58.4&36 \\
    + GDWS ($\beta=0.5$) &\bfu{82.53} & \bfu{50.96} & 50.6 & \bfu{102} \\
    \midrule
    VGG-16 ($p=90\%$) & 80.54& 49.44& 5.9&  36 \\
    + GDWS ($\beta=0.1$) &80.47& 49.52 & 31.5 & 93 \\
    \midrule
    VGG-16 ($p=95\%$) & 78.91& 48.74& 3.0&  36 \\
    + GDWS ($\beta=0.1$) &78.71& 48.53 & 18.3 & 106 \\
    \midrule
    VGG-16 ($p=99\%$) & 73.16& 41.74& 0.6&  41 \\
    + GDWS ($\beta=0.02$) &\bfu{72.75}& \bfu{41.56} & \bfu{2.9} & \bfu{136} \\
    \bottomrule
    \toprule
     WRN-28-4 (AT from \cite{sehwag2020hydra}) &85.35 & 57.23 & 22.3 &  17 \\
    + GDWS ($\beta=1$) &\bfu{84.17} &\bfu{55.87}& 20.5 & \bfu{68} \\
    \midrule
    WRN-28-4 ($p=90\%$) & 83.69& 55.20& 2.3&  17 \\
    + GDWS ($\beta=0.125$) &83.38& 54.79 &11.9 & 59 \\
    \midrule
    WRN-28-4 ($p=95\%$) & 82.68& 54.18& 1.1&  17 \\
    + GDWS ($\beta=0.005$) &82.59& 54.22 & 7.2 & 60 \\
    \midrule
    WRN-28-4 ($p=99\%$) & 75.62& 47.21& 0.2&  28 \\
    + GDWS ($\beta=0.0025$) &\bfu{75.36}& \bfu{47.04} & \bfu{1.2} & \bfu{68} \\
    \bottomrule
  \end{tabular}}
\end{wraptable}

\textbf{GDWS vs. Unstructured Pruning:} 
We compare GDWS with HYDRA \cite{sehwag2020hydra} which is an unstructured robust pruning method, on both CIFAR-10 and ImageNet datasets. We use the publicly released HYDRA models as well as their AT baselines, and apply GDWS to both the un-pruned and pruned models. Table~\ref{tab:comp-hydra} summarizes the robustness and FPS of HYDRA and GDWS networks on CIFAR-10. HYDRA pruned models have arbitrarily sparse weight matrices that cannot be leveraged by off-the-shelf hardware platforms immediately. Instead, we rely on the extremely high sparsity (99$\%$) of these matrices to emulate channel pruning whereby channels are discarded only if all filter weights are zero. This explains why, despite their high compression ratios, HYDRA models do not achieve significant improvements in FPS compared to their baselines.

For instance, a 99$\%$ HYDRA pruned WideResNet model achieves a massive $\sim$100$\times$ compression ratio and improves the FPS from 17 to 28, but suffers from a large $\sim$10$\%$ drop in both $\calA_\text{nat}$ and $\calA_\text{rob}$. In contrast, GDWS applied to the same un-pruned baseline preserves robustness and achieves significantly better throughput of 68 FPS, even though the model size reduction is negligible. Interestingly, we find that applying GDWS directly to HYDRA pruned models results in networks with high compression ratios with no robustness degradation and massive improvements in FPS compared to the pruned baseline. For example, applying GDWS to the same 99\% HYDRA pruned WideResNet achieves a $\sim$20$\times$ compression ratio and improves the throughput from 28 FPS to 68 FPS while preserving $\calA_\text{nat}$ and $\calA_\text{rob}$ of the pruned baseline. This synergy between HYDRA and GDWS is due to the fact that highly sparse convolution weight matrices are more likely to have low-rank and sparse sub-matrices. This implies that, using Lemma~\ref{lemm:gdws_mtx_concat}, sparse convolutions can be transformed to sparse GDWS versions with negligible approximation error. We explore this synergy in detail in the Appendix. Table~\ref{tab:comp-hydra-imagenet} shows that GDWS benefits also show up in ImageNet using ResNet-50.  

\begin{table}[hp]
  \caption{Comparison between HYDRA \cite{sehwag2020hydra} and GDWS using ResNet-50 on ImageNet.}
  \label{tab:comp-hydra-imagenet}
  \centering
  \resizebox{0.75\columnwidth}{!}{%
  \begin{tabular}{l|ccc|c}
    \toprule
    \bf{Models} & \bf{top-1 / 5 $\calA_{\text{nat}}\ [\%]$} & \bf{top-1 / 5 $\calA_{\text{rob}}\ [\%]$}    & \bf{Size $[\text{MB}]$} & \bf{FPS}\\
    \midrule
    ResNet-50 (AT from \cite{sehwag2020hydra}) &60.25 / 82.39 & 31.94 / 61.13 & 97.5 &15 \\
    + GDWS ($\beta=50$) &\bfu{58.04} / \bfu{80.56} & \bfu{30.22} / \bfu{58.48} & \bfu{86.2} &\bfu{19} \\
    \midrule
    ResNet-50 ($p=95\%$) & 44.60 / 70.12& 19.53 / 44.28& 5.1& 15 \\
    + GDWS ($\beta=0.5$) &43.91 / 69.46& 19.27 / 43.58& 12.6 &  19\\
    \midrule
    ResNet-50 ($p=99\%$) & 27.68 / 52.55& 11.32 / 28.83& 1.2& 17 \\
    + GDWS ($\beta=0.5$) & \bfu{26.27} / \bfu{50.90}& \bfu{10.92} / \bfu{27.55} & \bfu{2.9} & \bfu{25} \\
    \bottomrule
  \end{tabular}}
  \end{table}

\subsection{Defending against Union of Perturbation Models}\label{ssec:union}
Recent work has shown that adversarial training with a single perturbation model leads to classifiers vulnerable to the union of ($\ell_\infty, \ell_2, \ell_1$)-bounded perturbations \cite{schott2018towards,TB19,maini2020adversarial}. The method of multi steepest descent (MSD) \cite{maini2020adversarial} achieves state-of-the-art union robust accuracy ($\calA_\text{rob}^{\text{U}}$) against the union of ($\ell_\infty, \ell_2, \ell_1$)-bounded perturbations. We demonstrate the versatility of GDWS by applying it to a publicly available \cite{maini2020adversarial} robust pre(MSD)-trained ResNet-18 model on CIFAR-10. Following the setup in \cite{maini2020adversarial}, all attacks were run on a subset of the first 1000
test images with 10 random restarts with the following attack configurations: $\epsilon_\infty = 0.03$ with PGD-100 , $\epsilon_2 = 0.5$ with PGD-500, and $\epsilon_1 = 12$ with PGD-100. Table~\ref{tab:msd-results} shows that applying GDWS with $\beta = 0.01$ to the pre-trained ResNet-18 incurs a negligible ($<\sim$1$\%$) drop in $\calA_\text{nat}$ and $\calA_\text{rob}^{\text{U}}$ while improving the throughput from 28 FPS to 101 FPS ($>250\%$ improvement).

\begin{table}[htbp]
  \caption{Benefits of GDWS when evaluated against union of perturbation models on CIFAR-10. $\calA_\text{rob}^{\text{U}}$ is the fraction of test images that are simultaneously resistant to all perturbation models.}
  \label{tab:msd-results}
  \centering
    \resizebox{0.75\columnwidth}{!}{%
  \begin{tabular}{l|ccccc|c}
    \toprule
    \bf{Models} & \bf{$\calA_{\text{nat}}\ [\%]$} & \bf{$\calA_{\text{rob}}^{\infty}\ [\%]$}    & \bf{$\calA_{\text{rob}}^{1}\ [\%]$} & \bf{$\calA_{\text{rob}}^{2}\ [\%]$} & \bf{$\calA_{\text{rob}}^{\text{U}}\ [\%]$} & \bf{FPS}\\
    \midrule
    ResNet-18 (AT from \cite{maini2020adversarial}) &81.74 & 47.50 & 53.60 & 66.10& 46.10 & 28\\
    + GDWS ($\beta=0.0025$) & 81.67& 47.60 & 53.60 & 66.00 & 46.30 & 87\\
    + GDWS ($\beta=0.005)$ &81.43& 47.30& 52.60& 65.60& 45.70 & 92\\
    + GDWS ($\beta=0.01)$ &\bfu{81.10}& \bfu{47.20}& \bfu{52.20}& \bfu{65.00}& \bfu{45.20} & \bfu{101}\\
    \bottomrule
  \end{tabular}}
\end{table}

\section{Discussion}\label{sec:discussion}
We have established that the proposed GDWS convolutions are universal and efficient approximations of standard 2D convolutions that are able to accelerate any pre-trained CNN utilizing standard 2D convolution while preserving its accuracy and robustness. This facilitates the deployment of CNNs in safety critical edge applications where real-time decision making is crucial and robustness cannot be compromised. One limitation of this work is that GDWS alone does not achieve high compression ratios compared to pruning. Combining unstructured pruning with GDWS alleviates this problem to some extent. Furthermore, GDWS cannot be applied to CNNs utilizing DWS convolutions, such as RobNet for instance. An interesting question is to explore the possibility of training GDWS-structured networks from scratch. Another possible direction is fine-tuning post GDWS approximation to recover robustness, which we explore in the Appendix.

In summary, a GDWS approximated network inherits all the properties, e.g., accuracy, robustness, compression and others, of the baseline CNN while significantly enhancing its throughput (FPS) on real hardware. Therefore, the societal impact of GDWS approximated networks are also inherited from those of the baseline CNNs.

\begin{ack}
This work was supported by the Center for Brain-Inspired Computing (C-BRIC) and the Artificial Intelligence Hardware (AIHW) program funded by the Semiconductor Research Corporation (SRC) and the Defense Advanced Research Projects Agency (DARPA).
\end{ack}

\bibliographystyle{plain}
\bibliography{ref}
\medskip
\vfill
\pagebreak
\appendix
\section{Experimental Setup Details}\label{app:setup}
\subsection{Evaluation Setup}
In this section we provide details on how we measure FPS on the Jetson, as well as explain how we map GDWS convolutions efficiently. We use a single off-the-shelf NVIDIA Jetson Xavier NX developer kit for all our experiments. The Jetson Xavier is equipped with a 384-core NVIDIA Volta GPU, a 6-core NVIDIA Carmel ARM 64-bit CPU, and 8GB 128-bit LPDDR4x memory. We install the latest PyTorch packages onto the Jetson, as we will use their native neural network (NN) modules to implement both standard and GDWS convolutions. Specifically, we used PyTorch  v1.8.0 with Python v3.6.9 and CUDA v10.2.

\textbf{Measuring FPS:} The Python pseudo-code in \ref{lst:fps} explains how the FPS for any neural network model was measured on the Jetson. The main idea is to run successive inferences (batch size of 1) and measure the total elapsed time reliably, and calculate the FPS as the total number of inferences divided by the total elapsed time. To ensure consistency, we use 10000 inferences to measure FPS, after the GPU has been warmed up with 5000 inferences as well. Note that the measured FPS reflects the raw capabilities of the GPU, ignoring any I/O to and from the GPU.

\begin{lstlisting}[language=Python, caption=Python pseudo-code for measuring FPS on the Jetson using PyTorch modules., label={lst:fps}]

#get the appropriate NN architecture, e.g., ResNet-18
model =  get_architecture() 

#load the pre-trained model parameters from memory
model.load_state_dict(state_dict)

#transfer the model onto the GPU
model = model.cuda() 

#set the model in evaluation mode
model.eval() 

#sample a single test input and load it into GPU memory
x_test = get_input() #use batch size of 1
x_test.cuda()

#ensure no gradient overheads are introduced
with torch.no_grad(): 
    ## run successive inferences to warm-up the GPU
    for t in range(5000):
        y = model(x_test)
    
    ## setup synchronized timers in PyTorch
    start = torch.cuda.Event(enable_timing=True)
    end = torch.cuda.Event(enable_timing=True)
    start.record()
    
    ## now that the GPU is warmed-up, we run successive inferences and measure the total latency
    for t in range(num_inferences):
        y = model(x_test)
    
    ## measure the elapsed time delay in seconds
    end.record()
    torch.cuda.synchronize()
    delay = start.elapsed_time(end)/1000
    
## get the frames-per-second number
FPS = num_inferences/delay
\end{lstlisting}

\textbf{Mapping GDWS Convolutions:} Mapping GDWS convolutions requires mapping both the GDW and the PW convolutions efficiently onto the Jetson. PW layers are standard 2D convolutions with $1\times 1$ kernels, thus implementing PW convolutions using the PyTorch convolution module is straight forward. The challenge arises when mapping GDW convolutions, as it is a new convolutional structure that is not directly supported yet in PyTorch. To that end, we use simple tensor manipulations and leverage the existing support for standard DW convolution in PyTorch to implement GDW convolutions. 

Note that a $(C,K,\vc{g})$ GDW convolution operating on input tensor $\tnsr{X}\in \reals^{C\times H\times W}$convolves the $c^\text{th}$ input channel with $g_c\in\integers_+$ depthwise $K\times K$ filters to produce a total of $G$ intermediate output channels. A DW convolution operating on the same input tensor $\tnsr{X}$ is a special case of GDW where $g_c = 1 \ \forall c$. It is not difficult to see that a $(C,K,\vc{g})$ GDW convolution operating on $\tnsr{X}$ is equivalent to a DW convolution operating on the modified tensor $\tnsr{X}'\in \reals^{G\times H\times W}$ with $G = \sum g_c$ channels, where the tensor $\tnsr{X}'$ is obtained by duplicating the $c^\text{th}$ channel from $\tnsr{X}$ $g_c$ times. This tensor manipulation is implemented via simple tensor indexing in PyTorch. Therefore, we can efficiently map GDWS convolutions onto the Jetson without requiring any custom libraries. 

\subsection{Training Hyperparameters}
In the absence of any publicly available pre-trained models, we obtain strong baselines using AT \cite{madry2018towards} following the approach of \cite{rice2020overfitting} which utilizes early stopping to avoid robust over-fitting. We use the same hyperparameters, detailed below for our CIFAR-10 and SVHN baselines. A single workstation with two NVIDIA Tesla P100 GPUs is used for running all the training experiments.

\textbf{CIFAR-10:} For the CIFAR-10 experiments presented in Table~\ref{tab:ablation-cifar10-appndx} (Table~1 in main manuscript), we use PGD-7 adversarial training with $\epsilon = 8/255$ and step size $2/255$ for a maximum of 200 epochs and 128 mini-batch size. We employ a step-wise learning rate decay set initially at 0.1 and divided by 10 at epochs 100 and 150. We use a weight decay of $5\times 10^{-4}$, except for the lightweight networks  which were trained with a smaller weight decay of $2\times 10^{-4}$.

\textbf{SVHN:} For the SVHN experiments presented in Table~\ref{tab:svhn}, we use PGD-7 adversarial training with $\epsilon = 8/255$ and step size $2/255$ for a maximum of 200 epochs and 128 mini-batch size. We employ a step-wise learning rate decay set initially at 0.01 and divided by 10 at epochs 100 and 150. We use a weight decay of $5\times 10^{-4}$.

\subsection{Computing the Weight Error Vectors}
Constructing GDWS networks via Algorithm~\ref{alg:network} requires computing the per-layer weight error vectors $\{\bm{\alpha}_l\}$ as described in \eqref{eq:alpha} in Section~\ref{ssec:gdws-networks} of the main manuscript. Throughout all of our experiments, we compute the $\{\bm{\alpha}_l\}$ via an estimate of the mean over a small batch of \textit{adversarial} inputs sampled from the training set. Specifically, throughout all of experiments, we use 1000 input samples generated via PGD-7 with $\epsilon = 8/255$, except for ImageNet were 5000 adversarial input samples were used that were generated via PGD-4 with $\epsilon = 4/255$. 
\vfill
\pagebreak
\section{Additional Experiments and Comparisons}\label{app:experiments}
\subsection{Extended Ablation Study}
\textbf{Benefits of Non-uniform GDWS Networks:} 
We expand on Section~\ref{ssec:results} by comparing the benefits of using Algorithm~\ref{alg:network} vs. Algorithm~\ref{alg:approx1} to design GDWS networks. We denote networks obtained from Algorithm~\ref{alg:network} as GDWS-N (non-uniform reduction in complexity) and Algorithm~\ref{alg:approx1} as GDWS-U (uniform reduction in complexity). Specifically, for GDWS-U, we use Algorithm~\ref{alg:approx1}, with unweighted error ($\alpha_{c,l}=1$) to construct the error-optimal GDWS approximations of each layer, such that we reduce the number of MACs of each layer by the same fixed percentage.

We use VGG-16 on CIFAR-10 as our network and dataset of choice. We obtain different GDWS networks by varying the choice of $\beta$ in GDWS-N and the reduction percentage in GDWS-U. Figure~\ref{fig:ng_comp_stats} shows the per-layer reduction in MACs for both methods. As expected, GDWS-U produces uniform reductions across all layers, whereas GDWS-N is not restricted in that regard. In Figs~\ref{fig:ng_comp_nat} \& \ref{fig:ng_comp_rob} we compare both methods by plotting the natural and robust accuracies vs. FPS, respectively. The per-layer granularity inherit to GDWS-N allows it to outperform GDWS-U, as it consistently achieves higher natural and robust accuracies than at iso-FPS. 
\begin{figure}[bthp]
  \centering
    \subfloat[]{\includegraphics[height=4cm]{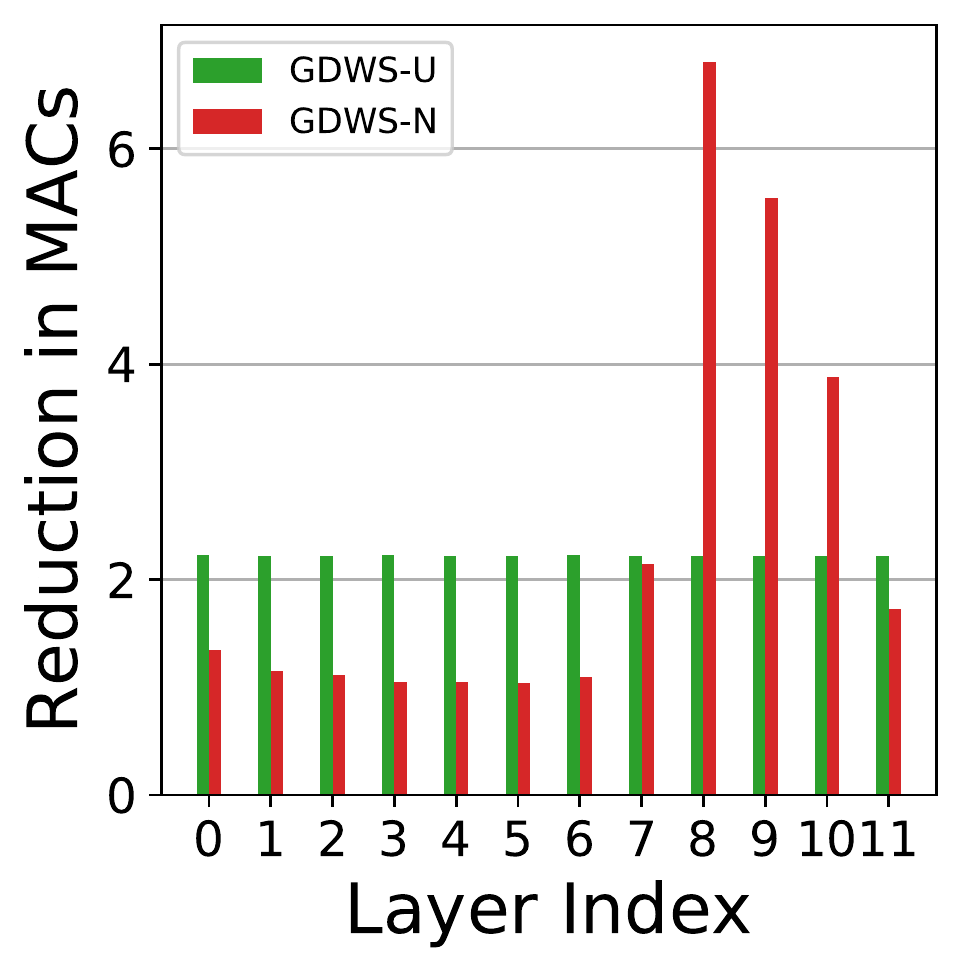}\label{fig:ng_comp_stats}}%
    \qquad%
    \subfloat[]{\includegraphics[height=4cm]{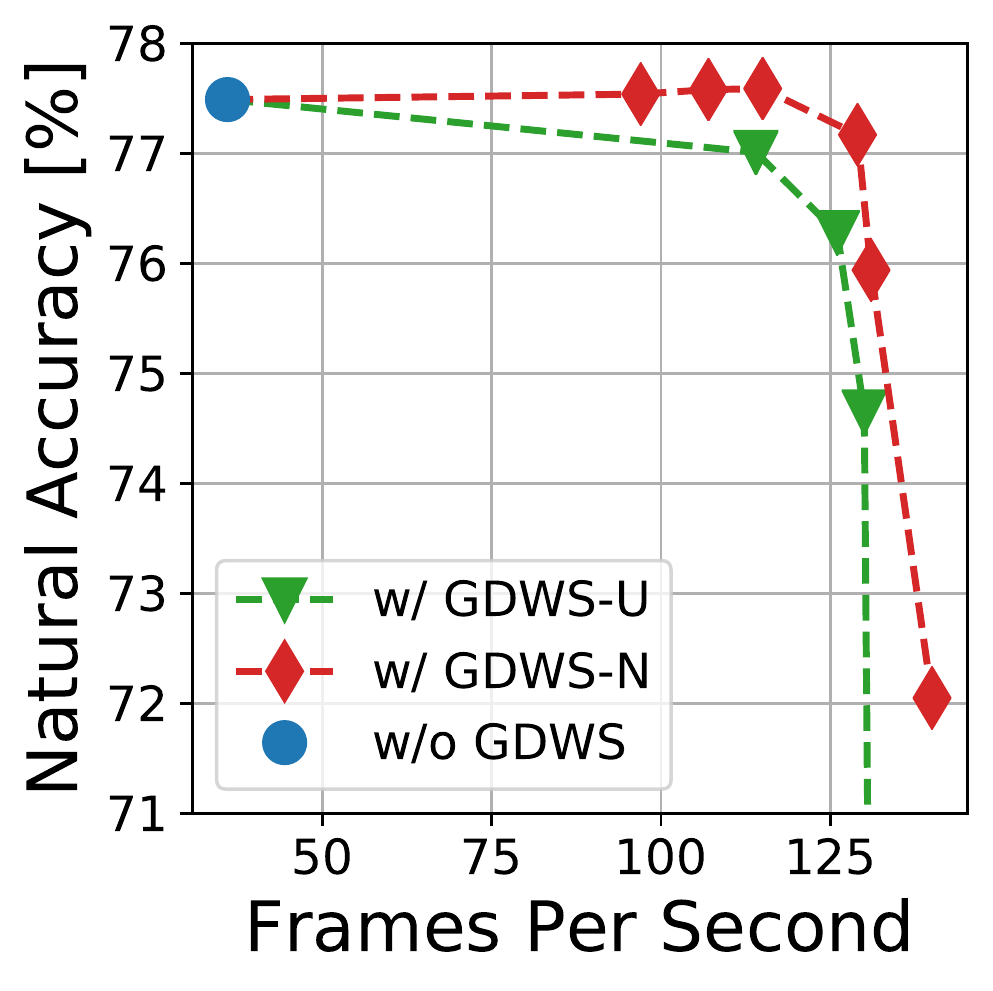}\label{fig:ng_comp_nat}}%
    \qquad%
    \subfloat[]{\includegraphics[height=4cm]{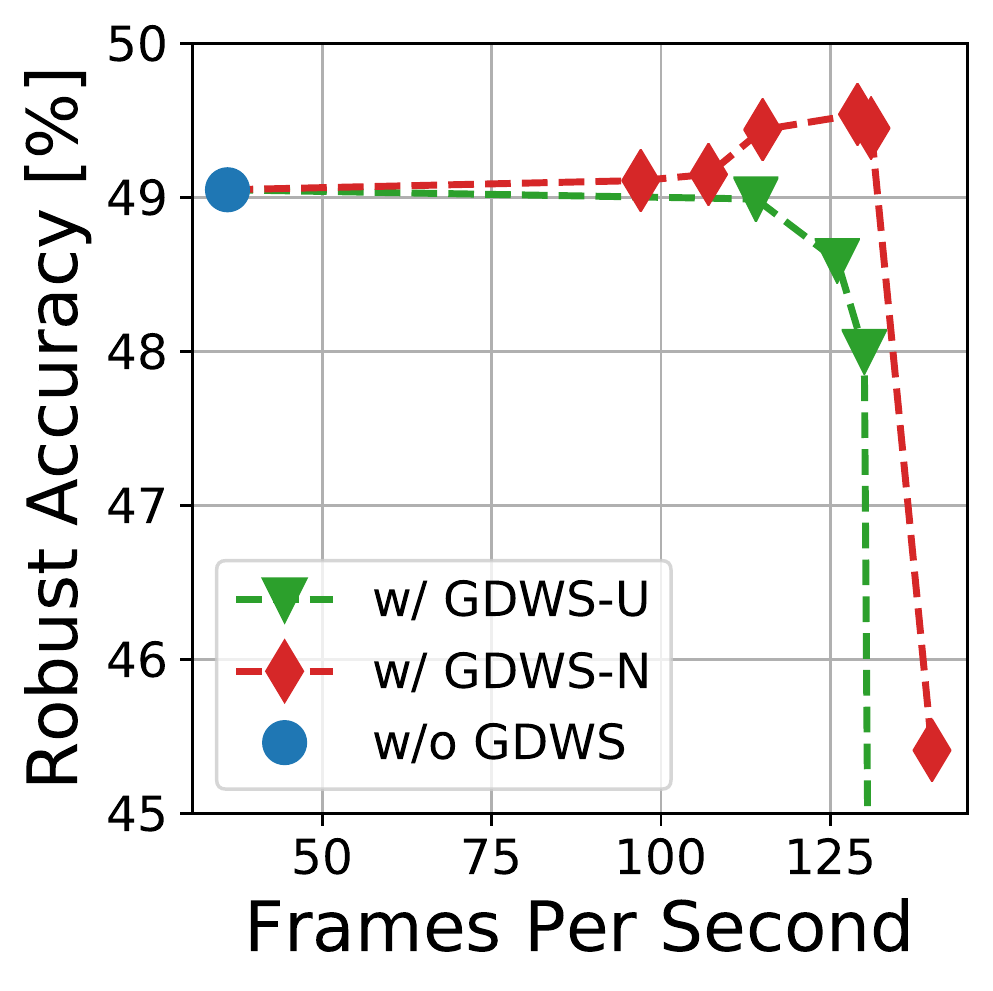}\label{fig:ng_comp_rob}}%
  \caption{Comparison of two methods for constructing GDWS networks using VGG-16 on CIFAR-10 by showing: (a) the per-layer reduction in complexity, (b) natural accuracy vs. FPS, and (c) robust accuracy vs. FPS.}
  \label{fig:ng_comp}
\end{figure}

\textbf{Impact of Fine-tuning:}  In this section, we showcase that fine-tuning via adversarial training for 10 epochs after the application of GDWS can significantly boost the efficacy of GDWS. In Table~\ref{tab:ablation-fine-tune}, we use the same VGG-16 baseline on CIFAR-10 from Table~\ref{tab:ablation-cifar10} in Section~\ref{ssec:results} and apply GDWS with higher approximation errors $\beta$. This results in GDWS networks with smaller model sizes and higher FPS, but with a significant degradation in robust and natural accuracies. As expected, fine-tuning boosts both robust and natural accuracies (up to $\sim1\%$ of the pre-trained baseline).

\begin{table}[hbpt]
\caption{Fine-tuning after GDWS using VGG-16 on CIFAR-10.}
\label{tab:ablation-fine-tune}
\centering
  \begin{tabular}{l|ccc|c}
    \toprule
    \bf{Models} & \bf{$\calA_{\text{nat}}\ [\%]$} & \bf{$\calA_{\text{rob}}\ [\%]$}    & \bf{Size $[\text{MB}]$}    & \bf{FPS}\\
    \midrule
    VGG-16     & 77.49 & 48.92 &   56.2  &   36  \\
    + GDWS ($\beta=0.25$) & 77.17 & 49.56 & 28.7 &129\\
    \midrule
    + GDWS ($\beta=2$) & 72.05  & 45.35  & 19.1 &140\\
    \ \ \ + fine-tune & \bfu{77.15} & \bfu{47.87} & 19.1 &140\\
    \midrule
   + GDWS ($\beta=5$) & 63.21  & 37.78  & 16.3 &143\\
   \ \ \ + fine-tune & \bfu{76.76} & \bfu{47.92} & 16.3 &143\\
    \bottomrule
  \end{tabular}
\end{table}

\textbf{Different Types of Attacks:} In this section, we conduct an extra set of attacks, highlighted in Table~\ref{tab:ablation-attacks} below, on the VGG-16 network on CIFAR-10 (same baseline as before). We use the Foolbox \cite{rauber2017foolbox} (\url{https://github.com/bethgelab/foolbox}) implementation of all these attacks to ensure proper implementation. All the attacks are using $\ell_\infty$-bounded perturbations with $\epsilon=8/255$, similar to our PGD results in the main manuscript. As expected, GDWS preserves the robustness of the pre-trained baseline, across different attack methods.

\begin{table}[hbpt]
\caption{Robustness across different types of attacks using VGG-16 on CIFAR-10.}
\label{tab:ablation-attacks}
\centering
  \begin{tabular}{l|ccc|c}
    \toprule
    \bf{Models} & \bf{$\calA_{\text{rob}}\ [\%]$ (FGSM)}    & \bf{$\calA_{\text{rob}}\ [\%]$ (BIM)}    & \bf{$\calA_{\text{rob}}\ [\%]$ (DeepFool)} & \bf{FPS}\\
    \midrule
    VGG-16     & 52.53 & 49.61 & 47.89 &   36  \\
  
   + GDWS ($\beta=0.25 $) &  53.19 & 50.08 & 47.28 & 129\\
   + GDWS ($\beta=0.5 $) & 52.69 & 49.87 & 46.32 & 131\\

    \bottomrule
  \end{tabular}
\end{table}

\textbf{Additional Results on CIFAR-10:}
This section expands on the CIFAR-10 results presented in Table~\ref{tab:ablation-cifar10} in Section~\ref{ssec:results} by adding additional GDWS data points with different values of $\beta$. Table~\ref{tab:ablation-cifar10-appndx} shows that GDWS networks preserve $\calA_\text{nat}$ and $\calA_\text{rob}$ as both are within $\sim$1$\%$ of their respective baselines. This further supports our claims in Section~\ref{ssec:results} that GDWS networks drastically improve the FPS while preserving robustness.

\begin{table}[hbpt]
\caption{Benefits of applying GDWS to standard pre-trained models on the CIFAR-10 dataset.}
\label{tab:ablation-cifar10-appndx}
\centering
  \begin{tabular}{l|ccc|c}
    \toprule
    \bf{Models} & \bf{$\calA_{\text{nat}}\ [\%]$} & \bf{$\calA_{\text{rob}}\ [\%]$}    & \bf{Size $[\text{MB}]$}    & \bf{FPS}\\
    \midrule
    ResNet-50 & 84.21 & 53.05 & 89.7 &16\\
    + GDWS ($\beta=0.001$) & 83.72 & 52.94 & 81.9 &37\\
    + GDWS ($\beta=0.005$) & 81.18 & 51.25 & 75.9 &39\\
    \midrule
    WRN-28-4  & 84.00 & 51.80     & 22.3 & 17  \\
    + GDWS ($\beta=5\times 10^{-6}$) & 83.64 & 51.62 & 19.9  &64\\
    + GDWS ($\beta=1\times 10^{-5}$) & 83.27 & 51.70 & 18.9 & 65\\
    \midrule
    ResNet-18 & 82.41  & 51.55 & 42.6 & 28    \\
    + GDWS ($\beta=0.001$) & 82.17 & 51.30 & 33.5 &89\\
    + GDWS ($\beta=0.005$) & 81.17 &50.98 & 29.1 & 104\\
    \midrule
    VGG-16     & 77.49 & 48.92 &   56.2  &   36  \\
    + GDWS ($\beta=0.1$) & 77.59 & 49.36 & 33.3 &115\\
   + GDWS ($\beta=0.25$) & 77.17 & 49.56 & 28.7 &129\\
    \bottomrule
  \end{tabular}
\end{table}

\textbf{New Results on SVHN:} Table~\ref{tab:svhn} shows that applying GDWS to pre-trained networks on SVHN maintains the robustness while offering significant improvements in FPS, which mirrors the same observations made on CIFAR-10.
\begin{table}[hbpt]
  \caption{Benefits of applying GDWS to standard pre-trained models on the SVHN dataset.}
  \label{tab:svhn}
  \centering
  \begin{tabular}{l|ccc|c}
    \toprule
    \bf{Models} & \bf{$\calA_{\text{nat}}\ [\%]$} & \bf{$\calA_{\text{rob}}\ [\%]$}    & \bf{Size $[\text{MB}]$} & \bf{FPS}\\
    \midrule
    WRN-28-4  & 90.71 & 52.27     & 22.3  &17  \\

    + GDWS ($\beta=0.0001$) & 90.67 & 51.89  &  22.3 & 56\\
    + GDWS ($\beta=0.0005$) & 90.60 & 51.11  & 22.1 & 64\\
    \midrule
    ResNet-18 & 88.63  & 55.57 & 42.6  & 28    \\
    + GDWS ($\beta=5\times 10^{-5}$) & 87.87 & 55.88 & 39.9  &80\\
    + GDWS ($\beta=7.5\times 10^{-5}$) & 87.37 & 55.66  & 39.3 & 89\\
    \midrule
    VGG-16     & 90.72 & 51.51 &   56.2  &   36  \\
    + GDWS ($\beta=0.1$) & 90.62 &  51.84  & 53.6 & 93\\
    + GDWS ($\beta=5$) & 88.09 & 54.48 & 43.3 & 125\\
    \bottomrule
  \end{tabular}
\end{table}

\subsection{Additional Comparisons with HYDRA}
In this section, we expand on the HYDRA \cite{sehwag2020hydra} comparison in Section~\ref{ssec:results} by: 1) providing additional GDWS networks obtained with different values of $\beta$ presented in Table~\ref{tab:comp-hydra-appndx}, 2) offering more insight to why HYDRA pruned networks achieve limited FPS improvement compared to their un-pruned baselines, and 3) explaining why GDWS accelerates HYDRA pruned networks without any loss in robustness. 

As seen in Section~\ref{ssec:results}, Table~\ref{tab:comp-hydra-appndx} shows that HYDRA pruned models do not achieve significant improvements in FPS compared to their un-pruned baselines. The reason is that, despite having arbitrarily sparse weight matrices, the filter sparsity is actually quite low. That is the number of prunable channels in HYDRA pruned models is small, especially for pruning ratios less than 95$\%$. To further demonstrate that effect, Figs~\ref{fig:stats-prune-vgg}~\&~\ref{fig:stats-prune-wrn} plot the per-layer filter sparsity of HYDRA pruned VGG-16 and WideResNet-28-4, respectively. These models are obtained from the publicly released CIFAR-10 HYDRA trained models available on GitHub. The plots indicate that only at \textit{extreme} pruning ratios such as 99$\%$ does the filter sparsity in both networks appear to be significant, which translates to some improvement in FPS on the Jetson.

\begin{table}
    \caption{Comparison between HYDRA \cite{sehwag2020hydra} and GDWS using VGG-16 and WRN-28-4 on CIFAR-10.}
      \label{tab:comp-hydra-appndx}
      \centering
  \begin{tabular}{l|ccc|c}
    \toprule
    \bf{Models} & \bf{$\calA_{\text{nat}}\ [\%]$} & \bf{$\calA_{\text{rob}}\ [\%]$}    & \bf{Size $[\text{MB}]$} &\bf{FPS}\\
    \midrule
    VGG-16 (AT from \cite{sehwag2020hydra})&82.72 & 51.93 & 58.4&36 \\
    + GDWS ($\beta=0.1$) &82.57 & 51.48 & 56.5 & 82 \\
    + GDWS ($\beta=0.5$) &82.53 &50.96 & 50.6 & 102 \\
    + GDWS ($\beta=1.2$) &81.41 & 47.88 & 44.2 & 111 \\
    \midrule
    VGG-16 ($p=90\%$) & 80.54& 49.44& 5.9&  36 \\
    + GDWS ($\beta=0.1$) &80.47& 49.52 & 31.5 & 93 \\
    + GDWS ($\beta=2$) &78.52& 47.26 & 26.9 & 101 \\
    \midrule
    VGG-16 ($p=95\%$) & 78.91& 48.74& 3.0&  36 \\
    + GDWS ($\beta=0.1$) &78.71& 48.53 & 18.3 & 106 \\
    + GDWS ($\beta=0.5$) &77.43& 46.99 & 17.1 & 117 \\
    \midrule
    VGG-16 ($p=99\%$) & 73.16& 41.74& 0.6&  41 \\
    + GDWS ($\beta=0.01$) &72.88& 41.79 & 3.0 & 130 \\
    + GDWS ($\beta=0.02$) &72.75& 41.56 & 2.9 & 136 \\
    \bottomrule
    \toprule
     WRN-28-4 (AT from \cite{sehwag2020hydra}) &85.35 & 57.23 & 22.3 &  17 \\
    + GDWS ($\beta=0.01$) &85.33 & 57.23 & 22.3 &53 \\
    + GDWS ($\beta=0.5$) &84.90 & 56.74 & 21.5 & 61 \\
    + GDWS ($\beta=1$) &84.17 &55.87& 20.5 & 68 \\
    \midrule
    WRN-28-4 ($p=90\%$) & 83.69& 55.20& 2.3&  17 \\
    + GDWS ($\beta=0.125$) &83.38& 54.79 &11.9 & 59 \\
    + GDWS ($\beta=0.4$) &81.21& 52.01 &11.4 & 65 \\
    \midrule
    WRN-28-4 ($p=95\%$) & 82.68& 54.18& 1.1&  17 \\
    + GDWS ($\beta=0.005$) &82.59& 54.22 & 7.2 & 60 \\
    + GDWS ($\beta=0.1$) &80.98& 52.60 & 6.9 & 65 \\
    \midrule
    WRN-28-4 ($p=99\%$) & 75.62& 47.21& 0.2&  28 \\
    + GDWS ($\beta=0.001$) &75.46& 47.30& 1.3 & 66 \\
    + GDWS ($\beta=0.0025$) &75.36& 47.04 & 1.2 & 68 \\
    \bottomrule
  \end{tabular}
\end{table}
\begin{figure}[htbp]
  \centering
    \subfloat[]{\includegraphics[height=4.4cm]{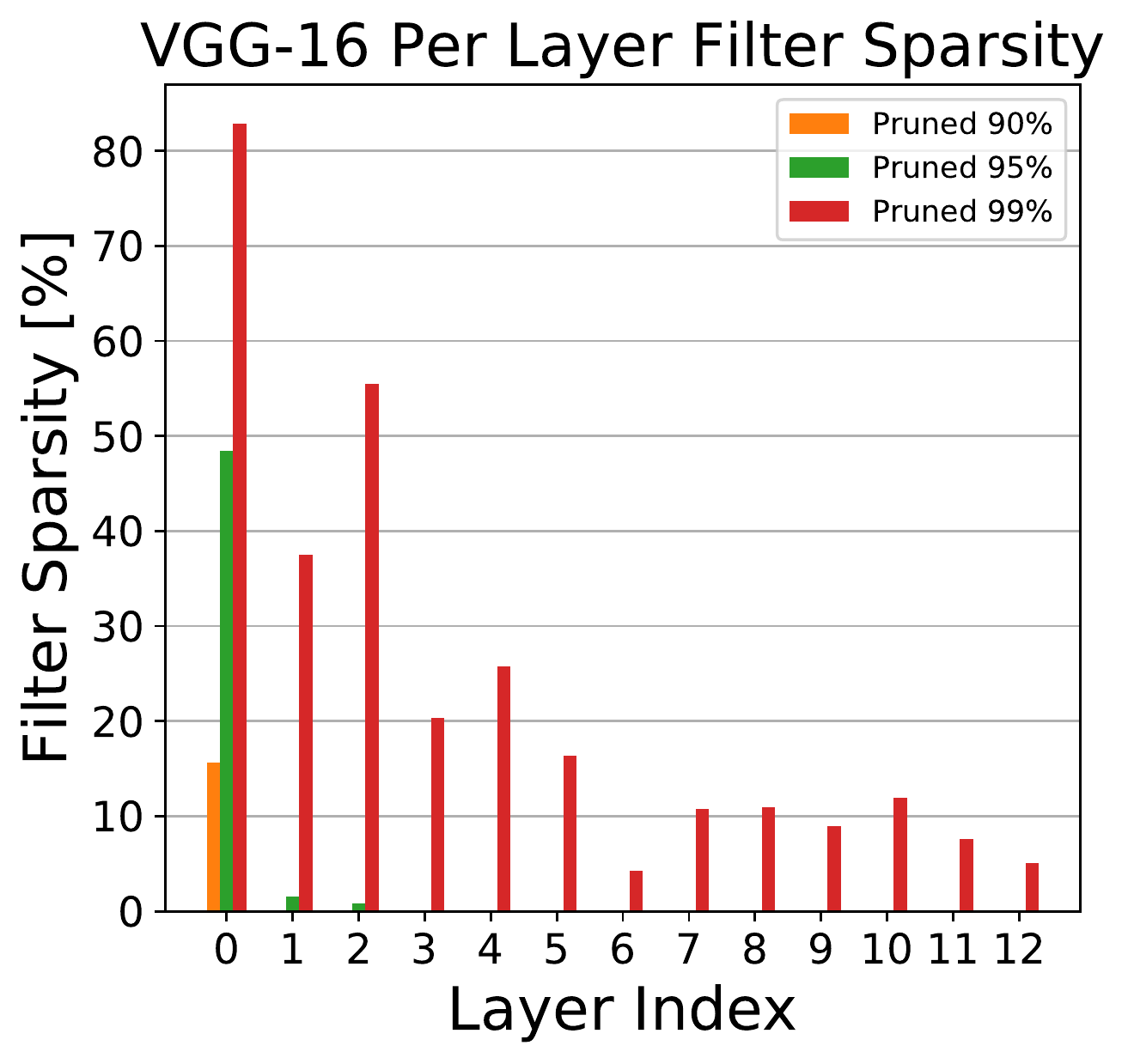}\label{fig:stats-prune-vgg}}%
    \qquad%
    \subfloat[]{\includegraphics[height=4.4cm]{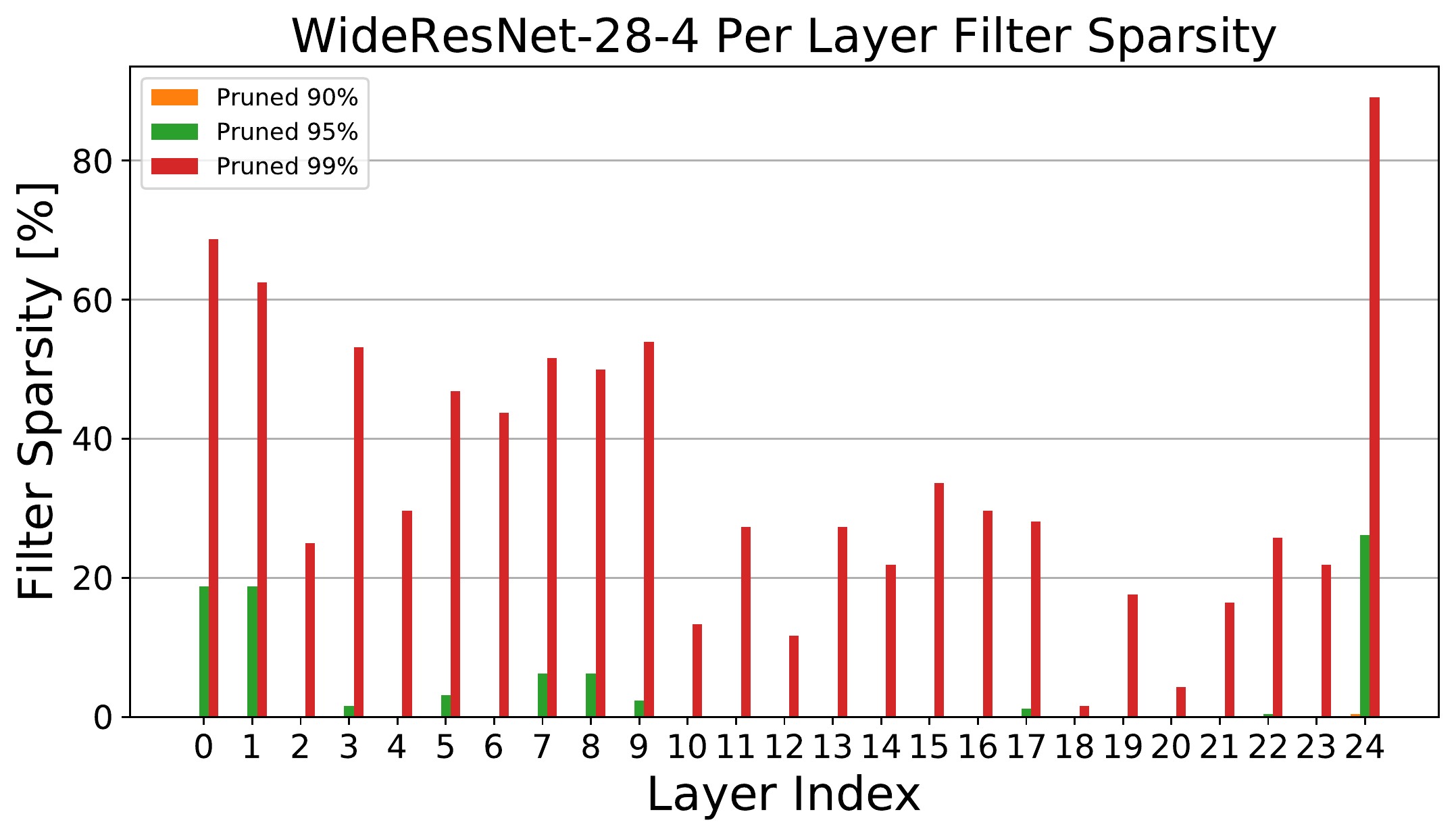}\label{fig:stats-prune-wrn}}%
  \caption{Per-layer filter sparsity of HYDRA pruned VGG-16 and WideResNet-28-4 models on CIFAR-10.}
\end{figure}
Table~\ref{tab:comp-hydra-appndx} also demonstrates that the application of GDWS to HYDRA-pruned networks provides significant improvement in FPS at iso-robustness when compared to the pruned baselines' numbers. Furthermore, the resultant GDWS networks are also sparse, which provide decent compression ratios. This synergy is due to the following observation: extremely sparse convolutional weight matrices have sub-matrices with low rank. This allows GDWS to transform standard sparse 2D convolutions into GDWS ones with no approximation error. The resultant GDWS convolutions are also sparse, which explains the improved compression ratios when compared to applying GDWS to un-pruned networks. To further understand this synergy, consider the following toy example: A standard $(3,2,4)$ 2D convolution with pruned weight matrix $\mtx{W} = [\mtx{W}_1|\mtx{W}_2|\mtx{W}_3]$:
\begin{equation}
    \mtx{W} = \begin{bmatrix} 
    w_1 & 0 & 0 & 0 &  0 & 0 & 0 & 0 & 0 & 0 & 0 & 0\\
    0 & 0 & 0 & 0 &  w_2 & 0 & 0 & 0 & 0 & 0 & 0 & 0\\ 
    0 & 0 & 0 & 0 &  0 & 0 & 0 & 0 & w_3 & 0 & 0 & 0\\ 
    0 & 0 & w_4 & 0 &  0 & 0 & 0 & 0 & 0 & 0 & 0 & 0\\ 
    \end{bmatrix}
\end{equation}
where:
\begin{equation}
    \mtx{W}_1 = \begin{bmatrix} 
    w_1 & 0 & 0 &0\\
    0 & 0 & 0 & 0 \\ 
    0 & 0 & 0 & 0\\ 
    0 & 0 & w_4 &0\\ 
    \end{bmatrix} \ \     \mtx{W}_2 = \begin{bmatrix} 
    0 & 0 & 0 &0\\
    w_2 & 0 & 0 & 0 \\ 
    0 & 0 & 0 & 0\\ 
    0 & 0 & 0 &0\\ 
    \end{bmatrix} \ \     \mtx{W}_3 = \begin{bmatrix} 
    0 & 0 & 0 &0\\
    0 & 0 & 0 & 0 \\ 
    w_3 & 0 & 0 & 0\\ 
    0 & 0 & 0 &0\\ 
    \end{bmatrix}
\end{equation}
and $w_i \neq 0 \ \forall i \in\{1,2,3,4\}$. Clearly, the weight matrx $\mtx{W}$ does not have an all zero row, which implies that the filter sparsity is zero, despite having a high sparsity rate of $100\times \frac{48-4}{48} = 91.\overline{6}\%$. However, we have that each sub-matrix $\mtx{W}_c$ has low rank. Specifically: $\text{rank}(\mtx{W}_1) = 2$ and $\text{rank}(\mtx{W}_2) = \text{rank}(\mtx{W}_3) = 1$, and computing the SVDs of each sub-matrix results in:
\begin{align}
\begin{split}
     \mtx{W}_1 &= \begin{bmatrix} 
    w_1 \\
    0  \\ 
    0 \\ 
    0 \\ 
    \end{bmatrix} \begin{bmatrix} 
    1 & 0 & 0 &0\\
    \end{bmatrix} + \begin{bmatrix} 
    0 \\
    0  \\ 
    0 \\ 
    w_4 \\ 
    \end{bmatrix} \begin{bmatrix} 
    0 & 0 & 1 &0\\
    \end{bmatrix} \\     \mtx{W}_2 &= \begin{bmatrix} 
    0 \\
    w_2  \\ 
    0 \\ 
    0 \\ 
    \end{bmatrix} \begin{bmatrix} 
    1 & 0 & 0 &0\\
    \end{bmatrix} \ \ \ \ \      \mtx{W}_3 = \begin{bmatrix} 
    0 \\
    0  \\ 
    w_3 \\ 
    0 \\ 
    \end{bmatrix} \begin{bmatrix} 
    1 & 0 & 0 &0\\
    \end{bmatrix}   
\end{split}
\end{align}
Thus, from Lemma~1, we can construct a $(3,2,\vc{g},4)$ GDWS convolution, where $\vc{g} = [2,1,1]^\text{T}$, without any approximation error. Decomposing into a GDW matrix and a PW matrix results in:
\begin{equation}
    \mtx{W}_\text{PW} = \begin{bmatrix} 
    w_1 & 0 & 0 &0\\
    0 & 0 & w_2 & 0 \\ 
    0 & 0 & 0 & w_3\\ 
    0 & w_4 & 0 &0\\ 
    \end{bmatrix} \ \     \mtx{W}_{\text{GDW}} = \begin{bmatrix} 
    1 & 0 & 0 & 0 &  0 & 0 & 0 & 0 & 0 & 0 & 0 & 0\\
    0 & 0 & 1 & 0 &  0 & 0 & 0 & 0 & 0 & 0 & 0 & 0\\ 
    0 & 0 & 0 & 0 &  1 & 0 & 0 & 0 & 0 & 0 & 0 & 0\\ 
    0 & 0 & 0 & 0 &  0 & 0 & 0 & 0 & 1 & 0 & 0 & 0\\ 
    \end{bmatrix}
\end{equation}
which are also sparse matrices. The total reduction in MACs is $\frac{4\times 4 + 4\times 4}{48} = 1.5\times$, and the total number of non-zero weights is $4+4=8$. This shows how extremely sparse standard 2D convolutions can be transformed into sparse GDWS convolutions with no approximation error while achieving improvements in complexity, which further justifies the synergy between GDWS and HYDRA pruned models.
\vfill
\pagebreak
\section{Proofs}\label{app:proofs}
In this section we provide proofs for Theorems~\ref{thm:approximation_comp} and \ref{thm:approximation_err} stated in Section~\ref{sec:gdws}. We first state the following result due to Eckart and Young \cite{eckart1936approximation} on low-rank matrix approximations:
\setcounter{lemma}{1}
\begin{lemma2}[Eckart-Young]\label{thm:ey}
Let $\mtx{A} \in \reals^{m\times n}$ be an arbitrary rank $r$ matrix with the singular value decomposition  $\mtx{A}=\mtx{U}\bm{\Sigma} \mtx{V}^{\text{\normalfont T}} = \sum_{i=1}^r\sigma_i \vc{u}_i \vc{v}_i^{\text{\normalfont T}}$, such that $\sigma_1 \geq \sigma_2 \geq ... \geq \sigma_r > 0$. Define for all $p\in\{1,2,...,r-1\}$\footnote{when $p=r$, the summation in \eqref{eq:ey-err} becomes undefined, but the error is zero.} the matrix $\hat{\mtx{A}}_p$:
\begin{equation} \label{eq:hard-thresh}
    \hat{\mtx{A}}_p = \sum_{i=1}^p\sigma_i \mtx{u}_i \vc{v}_i^{\text{\normalfont T}}
\end{equation}
Then $\hat{\mtx{A}}_p$ is the optimal rank $p$ approximation in both the following senses:
\begin{equation}
    \min_{\mtx{B}, \text{\normalfont rank}(\mtx{B})\leq p} ||\mtx{A}-\mtx{B}||_2 = \sigma_{p+1}
\end{equation}
\begin{equation}\label{eq:ey-err}
    \min_{\mtx{B}, \text{\normalfont rank}(\mtx{B})\leq p} ||\mtx{A}-\mtx{B}||_\text{\normalfont F} = \sqrt{\sum_{i=p+1}^r\sigma_{i}^2}
\end{equation}
\end{lemma2}
The Eckart-Young Lemma states that the truncated SVD can be used to compute the optimal rank $p$ approximation of any matrix in both the Frobenius norm and spectral norm sense. It also provides a closed form expression for the approximation error in terms of the singular values of the original matrix $\mtx{A}$. 
\subsection{Proof of Lemma~\ref{lemm:gdws_mtx_concat}}
\begin{lemma2}\label{lemm:gdws_mtx_concat-r} The weight matrix $\mtx{W}\in \reals^{M\times CK^2}$ of a $(C,K,\vc{g},M)$ GDWS convolution with $M>K^2$ can be expressed as the concatenation $[\mtx{W}_1|...|\mtx{W}_C]$ of $C$ sub-matrices $\mtx{W}_c\in \reals^{M\times K^2}$ such that $\text{\normalfont rank}(\mtx{W}_c) \leq \text{\normalfont{min}}(g_c,K^2)\ \forall c\in[C]$.
\end{lemma2}
\begin{proof}
From Property~\ref{prop:gdws_mtx}, $\mtx{W}$ is decomposed as:
\begin{equation}\label{eq:mtxw}
    \mtx{W} = [\mtx{W}_1|...|\mtx{W}_C] = \mtx{W}_\text{P} \mtx{W}_\text{D} = [\vc{u}_1 |...| \vc{u}_G]\times [\mtx{W}_{\text{D},1}| ... | \mtx{W}_{\text{D},C}]
\end{equation}
where $\vc{u}_i\in \reals^{M}$ are the column vectors of $\mtx{W}_\text{P}$ and $\mtx{W}_{\text{D},c}^{\text{T}} \in \reals^{K^2\times G}=[\vc{v}_{1,c}|...|\vc{v}_{G,c}]$ has column vectors $\vc{v}_{i,c} \in \reals^{K^2}$. From \eqref{eq:mtxw}, the sub-matrix $\mtx{W}_c\in \reals^{M\times K^2}$ is given by: 
\begin{align}\label{eq:sub-mtx}
    \mtx{W}_c = \mtx{W}_\text{P} \mtx{W}_{\text{D},c} = \sum_{i=1}^{G} \vc{u}_i \vc{v}^{\text{T}}_{i,c}
    = \sum_{i=1+h_c}^{h_c + g_c} \vc{u}_i \vc{v}^{\text{T}}_{i,c}
\end{align}
with $h_c=\sum_{k=1}^{c-1}g_k$ where we employ Property~2 to obtain the rightmost equality.
Therefore, $\mtx{W}_c$ is a sum of $g_c$ rank 1 matrices $\vc{u}_i \vc{v}^{\text{T}}_{i,c}\in\reals^{M\times K^2}$ which implies $\text{rank}(\mtx{W}_c) \leq \text{\normalfont{min}}(g_c,K^2)$. This concludes the proof.
\end{proof}

\subsection{Proof of Theorem~\ref{thm:approximation_comp}}
\textbf{Definition}: The \emph{weighted approximation error} between two matrices
$\mtx{W} = [\mtx{W}_1| ... | \mtx{W}_C]$ and $\mtx{Q} = [\mtx{Q}_1| ... | \mtx{Q}_C]$ is 
\begin{align}
    e(\mtx{W},\mtx{Q},\bm{\alpha}) &= \sqrt{\sum_{c=1}^C \alpha_c||\mtx{W}_c-\mtx{Q}_c||^2_\text{F}}
\end{align}
where $\bm{\alpha} \in \reals_+^C$ and all sub-matrices $\mtx{W}_c$ and $\mtx{Q}_c$ have the same size.

We first prove the following Lemma:
\begin{lemma}\label{lemma:approx}
Given any $(C,K,M)$ standard 2D convolution with weight matrix $\mtx{W} = [\mtx{W}_1| ... | \mtx{W}_C] \in\reals^{M\times CK^2}$,  $\mtx{W}_c \in \reals^{M\times K^2}$ and $K^2 < M$, the $(C,K,\vc{g},M)$ GDWS approximation with weight matrix $\hat{\mtx{W}}$ and fixed channel distribution vector $\vc{g}$ that minimizes the weighted approximation error $e(\mtx{W},\hat{\mtx{W}},\bm{\alpha})$ with $\bm{\alpha} \in \reals_+^C$ is obtained via the concatenation $\hat{\mtx{W}} = [\hat{\mtx{W}}_1 | ... | \hat{\mtx{W}}_C]$, where $\hat{\mtx{W}}_c$ is the optimal rank $g_c$ approximation of $\mtx{W}_c$. 
\end{lemma}
\begin{proof}
Since $\mtx{W}_c \in \reals^{M\times K^2}$ with $K^2 < M$, we have $\text{rank}(\mtx{W}_c) = r_c \leq K^2$. Let $\mtx{Q}= [\mtx{Q}_1| ... | \mtx{Q}_C]\in \reals^{M\times CK^2}$ be the weight matrix of a $(C,K,\vc{g},M)$ GDWS convolution. Then, from Lemma~\ref{lemm:gdws_mtx_concat}, we have that $\text{rank}(\mtx{Q}_c) \leq \text{\normalfont{min}}(g_c,K^2)$. Without loss of generality, we will always assume $g_c \leq r_c \leq K^2$, since otherwise $g_c > r_c$ for some $c$ implies the optimal rank $g_c$ approximation of $\mtx{W}_c$ is $\hat{\mtx{W}}_c = \mtx{W}_c$ resulting in at most $r_c$ non-zero DW kernels in the $c^\text{th}$ channel and $g_c - r_c$ zero DW kernels.

Then, from the Eckart-Young Lemma, we obtain:
\begin{equation}\label{eq:trncsvd}
    ||\mtx{W}_c-\mtx{Q}_c||_\text{\normalfont F} \geq \sqrt{\sum_{i=g_c+1}^{r_c}\sigma_{i,c}^2} = ||\mtx{W}_c-\hat{\mtx{W}}_c||_\text{\normalfont F}
\end{equation}
where $\sigma_{1,c} \geq \sigma_{2,c} \geq ... \geq \sigma_{r_c,c} > 0$ are the singular values of $\mtx{W}_c$ $\forall c\in [C]$ and $\hat{\mtx{W}}_c= \sum_{i=1}^{g_c}\sigma_{i,c} \vc{u}_{i,c} \vc{v}_{i,c}^{\text{\normalfont T}}$ is its rank $g_c$ truncated SVD.  The equality holds if and only if $\mtx{Q}_c =\hat{\mtx{W}}_c$. 

For a fixed $\vc{g}$, we have:
\begin{equation}
    e^2(\mtx{W},\mtx{Q},\bm{\alpha}) =  \sum_{c=1}^C \alpha_c||\mtx{W}_c-\mtx{Q}_c||^2_\text{F} \geq \sum_{c=1}^C \alpha_c\sum_{i=g_c+1}^{r_c}\sigma_{i,c}^2 = e^2(\mtx{W},\hat{\mtx{W}},\bm{\alpha})
\end{equation}
where $\hat{\mtx{W}} = [\hat{\mtx{W}}_1| ... | \hat{\mtx{W}}_C]$. This completes the proof since minimizing $e^2$ also minimizes $e$.
\end{proof}

We now prove Theorem~\ref{thm:approximation_comp}:
\begin{theorem2}\label{thm:approximation_compr}  Given a $(C,K,M)$ standard 2D convolution with weight matrix $\mtx{W}$, the $(C,K,\vc{g},M)$ GDWS approximation with weight matrix $\hat{\mtx{W}}$ that minimizes the error in \eqref{eq:err-exp} subject to $\sum g_c = G\leq \gamma$ (for some $\gamma \in \integers_+$), can be obtained in polynomial time via Algorithm~\ref{alg:approx1}.  
\end{theorem2}

\begin{proof}
We want to show that:
\begin{equation}\label{eq:thm1-restated}
    \hat{\mtx{W}} = \argmin_{\mtx{Q}:\ G\leq \gamma}  e(\mtx{W},\mtx{Q},\bm{\alpha})
\end{equation}
can be solved optimally for any $\bm{\alpha}\in \reals_+^{C}$. We show this using an induction on the constraint $\gamma$ via a constructive proof, which provides the basis for Algorithm~\ref{alg:approx1}. Essentially, we show that solving \eqref{eq:thm1-restated} with constraint $\gamma+1$ can be obtained from the solution of  \eqref{eq:thm1-restated} with constraint $\gamma$ via a 1D search over the channels $C$, and establish the base case for when $\gamma = 1$.

Without any loss of generality, we will assume that $\alpha_c >0\ \forall c\in [C]$. The reason for this is that if $\alpha_c = 0$ for a particular $c$, then we can set $\hat{\mtx{W}}_c = \vc{0}$ in the optimal solution and have $g_c = 0$ which minimizes the complexity and does not contribute to the error expression. Similar to before, let $\mtx{W} = [\mtx{W}_1| ... | \mtx{W}_C]$ be the concatenation of $C$ sub-matrices. We have $\text{rank}(\mtx{W}_c) = r_c \leq K^2$. Let the SVD of each sub-matrix be:
\begin{equation}
    \mtx{W}_c=\mtx{U}_c\bm{\Sigma}_c \mtx{V}_c^{\text{\normalfont T}} = \sum_{i=1}^{r_c}\sigma_{i,c} \vc{u}_{i,c} \vc{v}_{i,c}^{\text{\normalfont T}}
\end{equation}
where $\sigma_{1,c} \geq \sigma_{2,c} \geq ... \geq \sigma_{r_c,c} > 0$ are the singular values of $\mtx{W}_c$ $\forall c\in [C]$. 

Assume that $\hat{\mtx{W}}^{(\gamma)}$ is the optimal solution to \eqref{eq:thm1-restated} with constraint $\gamma$, that is $\hat{\mtx{W}}^{(\gamma)}$ corresponds to a $(C,K,\vc{g}^{(\gamma)},M)$ GDWS convolution with channel distribution vector $\vc{g}^{(\gamma)} \in \integers_+^C$ such that $G^{(\gamma)} = \sum g_c^{(\gamma)} \leq \gamma$. From Lemma~\ref{lemma:approx}, we have $\hat{\mtx{W}}^{(\gamma)} = [\hat{\mtx{W}}^{(\gamma)}_1| ... | \hat{\mtx{W}}^{(\gamma)}_C]$ such that $\text{rank}(\hat{\mtx{W}}^{(\gamma)}_c) \leq g^{(\gamma)}_c\leq r_c$, with optimal weighted approximation error:
\begin{equation}
    e^2(\mtx{W},\hat{\mtx{W}}^{(\gamma)},\bm{\alpha}) =  \sum_{c=1}^C \alpha_c\sum_{i=g^{(\gamma)}_c+1}^{r_c}\sigma_{i,c}^2
\end{equation}
Then, solving \eqref{eq:thm1-restated}, with constraint $G \leq \gamma +1$ will result in a $(C,K,\vc{g}^{(\gamma+1)},M)$ GDWS convolution such that the channel distribution vector $\vc{g}^{(\gamma+1)}$ will differ from $\vc{g}^{(\gamma)}$ in at most one position $c'\in [C]$, such that $g^{(\gamma +1)}_{c'}=g^{(\gamma)}_{c'} +1$. The reason for this is that: 1) $G^{(\gamma+1)} \geq G^{(\gamma)}$, otherwise the optimal solution for the $\gamma$ constraint could be improved; and 2) the integer constraints on both $G^{(\gamma+1)}$ and $G^{(\gamma)}$ imply that their difference can be at most 1, and hence the corresponding vectors will be identical up to one position. Thus, the optimal approximation error with constraint $\gamma + 1$ can be computed from $e^2(\mtx{W},\hat{\mtx{W}}^{(\gamma)},\bm{\alpha})$:
\begin{equation}
    e^2(\mtx{W},\hat{\mtx{W}}^{(\gamma+1)},\bm{\alpha}) =  e^2(\mtx{W},\hat{\mtx{W}}^{(\gamma)},\bm{\alpha}) - \max_{c \in [C]:  g^{(\gamma)}_{c} < r_c  }\alpha_c \sigma_{g^{(\gamma)}_{c}+1,c}^2
\end{equation}
where the maximization is taken over all channels $c$ that are not saturated (that is $g^{(\gamma)}_{c}+1 \leq r_c$ is valid). If no such channels exist, then the approximation error is saturated, and there is no point in increasing complexity further, which implies $\vc{g}^{(\gamma + 1)} = \vc{g}^{(\gamma)}$. 
Therefore, we can construct the optimal channel distribution vector $\vc{g}^{(\gamma +1)}$ from $\vc{g}^{(\gamma)}$ as previously mentioned, and then use Lemma~\ref{lemma:approx} to find $\hat{\mtx{W}}^{(\gamma +1)}$.

Lastly, we show how to solve \eqref{eq:thm1-restated} for the smallest constraint $\gamma=1$, which establishes the base case, and thus concludes the proof. Notice that, if $\gamma=1$, then $G=1$, and $\vc{g}$ reduces to the basis vector $\vc{e}_c$ (vector of all zeros except for one position $c$ such that $e_c=1$). Thus the optimal GDWS approximation with $G=1$ can be solved by simply searching for the channel $c$ that maximizes $\alpha_c \sigma_{1,c}^2$, and then use Lemma~\ref{lemma:approx} to find $\hat{\mtx{W}}^{(1)}$.
\end{proof}

\subsection{Proof of  Theorem~\ref{thm:approximation_err}}
\begin{theorem2}\label{thm:approximation_errr} 
Given a $(C,K,M)$ standard 2D convolution with weight matrix $\mtx{W}$, the $(C,K,\vc{g},M)$ GDWS approximation with weight matrix $\hat{\mtx{W}}$ that minimizes the complexity in \eqref{eq:gdws-comp} subject to $e(\mtx{W},\mtx{Q},\bm{\alpha})\leq \beta$ (for some $\beta \geq 0$), can be constructed in polynomial time via Algorithm~\ref{alg:approx2}. 
\end{theorem2}
\begin{proof}
We want to show that:
\begin{equation}\label{eq:thm2-restated}
    \hat{\mtx{W}} = \argmin_{\mtx{Q}:\ e(\mtx{W},\mtx{Q},\bm{\alpha})\leq \beta}  \sum_{c=1}^Cg_c
\end{equation}
can be solved for any weight error vector $\bm{\alpha}\in\reals_+^C$ in polynomial time. We show this by first applying a re-formulation of both the objective and the constraint as a function of a single binary vector. Using this new formulation, we show that solving \eqref{eq:thm2-restated} reduces to a greedy approach, captured in Algorithm~\ref{alg:approx2}, consisting of a simple 1D search over sorted quantities.

Similar to before, let $\mtx{W} = [\mtx{W}_1| ... | \mtx{W}_C]$ be the concatenation of $C$ sub-matrices. We have $\text{rank}(\mtx{W}_c) = r_c \leq K^2$. Let the SVD of each sub-matrix be:
\begin{equation}
    \mtx{W}_c=\mtx{U}_c\bm{\Sigma}_c \mtx{V}_c^{\text{\normalfont T}} = \sum_{i=1}^{r_c}\sigma_{i,c} \vc{u}_{i,c} \vc{v}_{i,c}^{\text{\normalfont T}}
\end{equation}
where $\sigma_{1,c} \geq \sigma_{2,c} \geq ... \geq \sigma_{r_c,c} > 0$ are the singular values of $\mtx{W}_c$ $\forall c\in [C]$. Furthermore, without loss of generality we will assume that $\alpha_{c}>0$ $\forall c \in [C]$. For a fixed channel distribution vector $\vc{g}$, weight error vector $\bm{\alpha}$ and convolution matrix $\mtx{W}$, Lemma~\ref{lemma:approx} states that the optimal GDWS approximation error can be computed via:
\begin{equation} \label{eq:gdws-err}
     e^2(\vc{g}) = e^2(\mtx{W},\hat{\mtx{W}},\bm{\alpha}) =  \sum_{c=1}^C \alpha_c\sum_{i=g_c+1}^{r_c}\sigma_{i,c}^2
\end{equation}
Therefore, for any $\beta \geq 0$, there always exists a GDWS convolution satisfying $e(\vc{g}) \leq \beta$. A simple choice of $g_c = r_c\ \forall c \in [C]$ will result in $e(\vc{r}) = 0 \leq \beta$, where $\vc{r}\in \integers_+^C$ is the vector of sub-matrix ranks $r_c$'s. The goal is to find the least complex GDWS convolution, satisfying the constraint. 

Let $\calA$ be an ordered set of all  $R = \sum r_c$ quantities $\alpha_c \sigma^2_{i,c}$. Define an indexing $k \in [R]$ on $\calA$ where $a_k \in \calA$ corresponds to a unique pair $(i,c)$ such that $a_k = \alpha_c \sigma^2_{i,c}$ and $a_1 \geq a_2 \geq ... \geq a_R > 0$. By doing so, we can re-write the error expression \eqref{eq:gdws-err}:
\begin{equation} \label{eq:gdws-err-2}
     e^2(\vc{g}) =  \sum_{c=1}^C \alpha_c\sum_{i=g_c+1}^{r_c}\sigma_{i,c}^2 = \sum_{k=1}^{R}a_k t_k
\end{equation}
where $t_k \in \{0,1\}$ are binary variables indicating whether the corresponding pair $(i,c)$ exists in the original sum in \eqref{eq:gdws-err}. This change of variables facilitates the optimization problem in \eqref{eq:thm2-restated}, since the binary vector $\vc{t}\in \{0,1\}^R$ can be used to enumerate all possible GDWS approximations with a simple expression of the optimal error in \eqref{eq:gdws-err-2}. Another useful thing about this re-formulation is the following property:
\begin{equation} \label{eq:gdws-complexity}
    \sum_{c=1}^C g_c = \sum_{k=1}^R\overline{t}_k = G
\end{equation}
where $\overline{t}_k = |1-t_k|$ is the flipped binary variable. 
Using the fact that the $\{a_k\}$'s are sorted in descending order, let $j \in [R]$ be the smallest index such that:
\begin{equation}
    \sum_{k=j+1}^{R} a_k \leq \beta^2
\end{equation}
Then setting $t_k = 1$ $\forall k > j$ and $t_k = 0$ otherwise, will result in the least complex (least sum $\sum \overline{t}_k$) GDWS approximation satisfying the error constraint $e\leq \beta$. Finding the index $j$ can be done via a simple 1D search, by starting with $j=R$ (corresponding to the zero error case), and keep decrementing $j$ until the error condition is no longer satisfied. After finding the optimal vector $\vc{t}$, the corresponding unique channel distribution vector can be constructed via the index mapping:
\begin{equation}
    g_c = \sum_{k \in \calK_c} \overline{t}_k
\end{equation}
where $\calK_c \subset [R]$ is the set of indices $k$ such that the corresponding index pair $(i,c')$ satisfies $c'=c$. Finally, given the channel distribution vector $\vc{g}$, we can use Lemma~\ref{lemma:approx} to construct $\hat{\mtx{W}}$. The greedy algorithm presented in Algorithm~\ref{alg:approx2} computes $\vc{g}$ via this approach, but without dealing with the auxiliary indexing and reformulation. 
\end{proof}
\vfill
\pagebreak
\section{Rationale for the Weight Error Vector Expression}\label{app:sensitivity}

In this section, we provide a detailed explanation for our choice of $\bm{\alpha}_l$ in \eqref{eq:alpha}. The work of \cite{sakr2017analytical} presents theoretical bounds on the accuracy of neural networks, in the presence of quantization noise due to quantizing both weights and activations, to determine the minimum precision required to maintain accuracy. A follow-up work \cite{sakr2018analytical} extends this bound to the per-layer precision case, allowing for better complexity-accuracy trade-offs. The bound in \cite{sakr2017analytical} in fact is much more general, and is not restricted to neural network quantization. Consider the following \textit{scalar} additive perturbation model:
\begin{equation} \label{eq:model}
    \hat{w} = w + \eta_w 
\end{equation}
where $\eta_w$ is assumed to be a zero-mean, symmetric and independently distributed scalar random variable with variance $s^2$. Then the work of \cite{sakr2017analytical,sakr2018analytical} shows that the probability $p_{\text{m}}$ that the noisy network $\hat{f}$ paramerterized by $\hat{w}$ differs in its decision from $f$ can be upper bounded as follows:
\begin{equation}\label{eq:mismatch}
    p_{\text{m}} \leq \sum_{l = 1}^{L} \sum_{c = 1}^{C_l}s^2_{c,l} E_{c,l} \ \ \ \ \text{where} \ \ \ \ E_{c,l} = \mathbb{E}\left[ \sum_{\substack{j=1\\ j \neq n_x}}^N \frac{\sum_{w \in \calW^{(l)}_c}\left|\frac{\partial \delta_{x,j}}{\partial w}\right|^2}{2\delta_{x,j}^2}\right]
\end{equation}
where the following notation, inherited from Section~\ref{ssec:gdws-networks}, is used: Let $f: \reals^D \rightarrow \reals^{N}$ be a pre-trained CNN for an $N$-way classification problem with $L$ convolutional layers parameterizd by weight matrices $\mtx{W}^{(l)} \in \reals^{M_l\times C_lK_l^2}$. The CNN $f$ operates on a $D$-dimensional input vector $\vc{x}$ to produce a vector $\vc{z}=f(\vc{x})$ of soft outputs or logits. Denote by $n_x \in [N]$ the predicted class label associated with $\vc{x}$, and define $\delta_{x,j} = z_j - z_{n_x}$ to be the soft output differences $\forall j\in[N]\setminus\{n_x\}$. \textbf{In addition}, define $\calW^{(l)}_c$ to be the set containing all the $C_lK_l^2$ scalar entries of sub-matrix  $\mtx{W}^{(l)}_c$ $\forall c \in [C_l]$ $\forall l \in [L]$, that is the cardinality of $\calW^{(l)}_c$ is $C_lK_l^2$. Using this notation, $\calW^{(l)} = \bigcup_c \calW^{(l)}_c$ is essentially the set of all scalar parameters in the $l^{\text{th}}$ convolutional layer, and $\calW = \bigcup_l \calW^{(l)}$ is the set of all scalar parameters of $f$ across all convolutional layers.

When approximating standard 2D convolutions with GDWS convolutions, we incur approximation errors that are captured at the sub-matrix level, and not at the entry level. Let $\mtx{W}^{(l)} = [\mtx{W}^{(l)}_1|\mtx{W}^{(l)}_2| ... |\mtx{W}^{(l)}_C]$ be the weight matrix, and its corresponding sub-matrices, of the standard convolution for layer $l$. Define $r_c^{(l)} = \text{rank}(\mtx{W}^{(l)}_c) \leq {K_l}^2$. Similarly, let $\mtx{Q}^{(l)} = [\mtx{Q}^{(l)}_1|\mtx{Q}^{(l)}_2| ... |\mtx{Q}^{(l)}_C]$ be the weight matrix, and its corresponding sub-matrices, of the GDWS convolution approximation for layer $l$. From Lemma~\ref{lemm:gdws_mtx_concat} we know that $\text{rank}(\mtx{Q}^{(l)}_c) = g_c^{(l)}$. Then, based on the proofs in Appendix~\ref{app:proofs}, the sub-matrix approximation error can be expressed as:
\begin{equation}
    e^{(l)}_c = ||\mtx{W}^{(l)}_c - \mtx{Q}^{(l)}_c||_\text{F}  =||\mtx{R}^{(l)}_c||_\text{F} =  \sqrt{\sum_{i=g_c^{(l)}+1}^{r^{(l)}_c}{\sigma_{i,c}^{(l)}}^2}
\end{equation}
where $\sigma_{1,c}^{(l)} \geq \sigma_{2,c}^{(l)} \geq ... \geq \sigma_{r_c^{(l)},c}^{(l)} > 0$ are the singular values of $\mtx{W}^{(l)}_c$ $\forall c\in [C_l]$ $\forall l\in[L]$. Clearly, the setup in \cite{sakr2017analytical,sakr2018analytical} does not hold here. However, we circumvent this issue by assuming that for all entries $w \in \calW^{(l)}_c$, the additive perturbation model in \eqref{eq:model} holds where $\eta_w$ are additive, zero-mean, symmetric, independent random variables with variance:
\begin{equation}
    \mathbb{E}\left[{\eta_w}^2\right] = \frac{||\mtx{R}^{(l)}_c||^2_\text{F}}{M_lK_l^2} \ \ \ \ \forall c\in[C_l] \ \forall l \in [L]
\end{equation}
While this assumption does not hold, it allows us to use the upper bound in \eqref{eq:mismatch} to provide a \textbf{heuristic} in our setup:
\begin{align}
    \begin{split}
        p_{\text{m}} &\leq \sum_{l = 1}^{L} \sum_{c = 1}^{C_l}s^2_{c,l} E_{c,l}  \\
        &= \sum_{l = 1}^{L} \sum_{c = 1}^{C_l}\frac{||\mtx{R}^{(l)}_c||^2_\text{F}}{M_lK_l^2}\ \mathbb{E}\left[ \sum_{\substack{j=1\\ j \neq n_x}}^N \frac{\sum_{w \in \calW^{(l)}_c}\left|\frac{\partial \delta_{x,j}}{\partial w}\right|^2}{2\delta_{x,j}^2}\right] \\
        &= \sum_{l = 1}^{L} \sum_{c = 1}^{C_l}\frac{||\mtx{R}^{(l)}_c||^2_\text{F}}{M_lK_l^2}\ \mathbb{E}\left[ \sum_{\substack{j=1\\ j \neq n_x}}^N \frac{||\mtx{D}^{(c,l)}_{x,j}||^2_\text{F}}{2\delta_{x,j}^2}\right] \\
        &= \sum_{l = 1}^{L} \sum_{c = 1}^{C_l} \alpha_{c,l}  ||\mtx{R}^{(l)}_c||^2_\text{F} \\
        &= \sum_{l = 1}^{L} e(\mtx{W}^{(l)}, \mtx{Q}^{(l)},\bm{\alpha}_l)^2
    \end{split}
\end{align}
where $\alpha_{c,l}$ is the same as before, with the definition $\mtx{D}^{(c,l)}_{x,j} \in \reals^{M_l \times K_l^2}$ being the derivative of $\delta_{x,j}$ w.r.t. the sub-matrix $\mtx{W}^{(l)}_c$. Thus, the upper bound on $ p_{\text{m}}$ results in a sum of $L$ terms, where each term is the GDWS approximation error. Following \cite{sakr2018analytical}, we use noise gain equalization to minimize this sum. That is we make sure all the terms are of comparable magnitude by upper-bounding them with the same $\beta$ when using Algorithm~\ref{alg:approx2}.
\end{document}